\documentclass[draft]{eptcs}
\providecommand{\event}{TARK2017}
\usepackage{underscore}   

\usepackage[utf8]{inputenc}
\usepackage{amsmath,amssymb,graphicx}
\usepackage{url}
\usepackage{prooftrees}
\renewcommand{\phi}{\varphi}
\renewcommand{\epsilon}{\varepsilon}
\newcommand{\M}{\mathcal{M}}
\newcommand{\N}{\mathcal{N}}
\newcommand{\T}{\mathcal{T}}
\newcommand{\END}{\tt{END}}

\newcommand{\lra}{\leftrightarrow}

\newcommand{\BR}{\tt{BR}}
\newcommand{\K}{\ensuremath{\K}}
\newcommand{\Kh}{\ensuremath{\mathit{Kh}}}

\newcommand{\Kv}{\ensuremath{\mathit{Kv}}}

\newcommand{\hK}{\ensuremath{\widehat{\K}}}
\renewcommand{\hK}{\ensuremath{\Diamond}}
\newcommand{\hKS}{\ensuremath{\hK}^x}
\newcommand{\hKSy}{\ensuremath{\hK}^y}
\newcommand{\FOML}{\ensuremath{\mathbf{FOML}^\approx}}
\newcommand{\SFOL}{\ensuremath{\mathbf{2SFOL}^\approx}}
\newcommand{\FOL}{\ensuremath{\mathbf{FOL}^\approx}}
\newcommand{\PSPACE}{\textsc{PSPACE}}

\newcommand{\KA}{\ensuremath{\blacksquare^{x}}}

\newcommand{\KS}{\ensuremath{\K^{x}}}
\newcommand{\KSy}{\ensuremath{\K^{y}}}
\newcommand{\KSxf}{\ensuremath{\K^{x_1}}}
\newcommand{\KSxl}{\ensuremath{\K^{x_n}}}

\newcommand{\hKSyf}{\ensuremath{\hK^{y_1}}}
\newcommand{\hKSyl}{\ensuremath{\hK^{y_m}}}

\newcommand{\SMLMSK}{\ensuremath{\mathsf{SMLMSK}}}
\newcommand{\SMLMSKEQ}{\ensuremath{\mathsf{SMLMSK}^\approx}}
\newcommand{\MLMS}{\ensuremath{\mathbf{MLMS}}}

\newcommand{\MLMSK}{\ensuremath{\mathbf{MLMSK}}}
\newcommand{\MLMSKEQ}{\ensuremath{\mathbf{MLMSK}^\approx}}

\newcommand{\MLMSEQ}{\ensuremath{\mathbf{MLMS}^{\approx}}}

\newcommand{\lr}[1]{\langle #1 \rangle}

\newcommand{\Ps}{\textbf{P}}
\newcommand{\Qs}{\textbf{Q}}
\newcommand{\Var}{\textbf{X}}
\newcommand{\War}{\textbf{Y}}

\newcommand{\ef}{\ensuremath{\exists\Box}}

\newcommand{\TAUT}{\ensuremath{\mathtt{TAUT}}}
\newcommand{\NECK}{\ensuremath{\mathtt{NECK}}}
\newcommand{\NECMS}{\ensuremath{\mathtt{NECMS}}}

\newcommand{\DISTK}{\ensuremath{\mathtt{DISTK}}}
\newcommand{\AxTrK}{\ensuremath{\mathtt{T}}}
\newcommand{\AxTrans}{\ensuremath{\mathtt{4MS}}}
\newcommand{\AxEuc}{\ensuremath{\mathtt{5MS}}}
\newcommand{\AxTransK}{\ensuremath{\mathtt{4}}}
\newcommand{\AxEucK}{\ensuremath{\mathtt{5}}}

\newcommand{\MP}{\ensuremath{\mathtt{MP}}}
\newcommand{\MONO}{\ensuremath{\mathtt{MONOMS}}}
\newcommand{\RE}{\ensuremath{\mathtt{RE}}}
\newcommand{\AxKtoMS}{\ensuremath{\mathtt{KtoMS}}}
\newcommand{\AxMStoMSK}{\ensuremath{\mathtt{MStoMSK}}}
\newcommand{\AxMSKtoMS}{\ensuremath{\mathtt{MSKtoMS}}}
\newcommand{\AxMStoK}{\ensuremath{\mathtt{MStoK}}}

\newcommand{\KtoMS}{\ensuremath{\mathtt{RKtoMS}}}

\newcommand{\AxId}{\ensuremath{\mathtt{ID}}}
\newcommand{\AxSym}{\ensuremath{\mathtt{SYM}}}
\newcommand{\AxTranseq}{\ensuremath{\mathtt{TRANS}}}
\newcommand{\AxKEQ}{\ensuremath{\mathtt{KEQ}}}
\newcommand{\AxKNEQ}{\ensuremath{\mathtt{KNEQ}}}
\newcommand{\AxMStotK}{\ensuremath{\mathtt{MStotK}}}

\newcommand{\AxSub}{\ensuremath{\mathtt{SUBID}}}

\newcommand{\AxRMS}{\ensuremath{\mathtt{RMS}}}
\newcommand{\AxKT}{\ensuremath{\mathtt{KT}}}
\newcommand{\AxMST}{\ensuremath{\mathtt{MST}}}

\renewcommand{\iff}{\Leftrightarrow}

\renewcommand{\K}{\Box}
\renewcommand{\vec}{\overline}

\newcommand{\bis}{\mathrel{\mathchoice%
{\raisebox{.3ex}{$\,
  \underline{\makebox[.7em]{$\leftrightarrow$}}\,$}}%
{\raisebox{.3ex}{$\,
  \underline{\makebox[.7em]{$\leftrightarrow$}}\,$}}%
{\raisebox{.2ex}{$\,
  \underline{\makebox[.5em]{\scriptsize$\leftrightarrow$}}\,$}}%
{\raisebox{.2ex}{$\,
  \underline{\makebox[.5em]{\scriptsize$\leftrightarrow$}}\,$}}}}

\usepackage{color}
\usepackage[all]{xy}
\newtheorem{definition}{Definition}[section]
\newtheorem{proposition}{Proposition}[section]

\newtheorem{corollary}{Corollary}[section]
\newtheorem{theorem}{Theorem}[section]
\newtheorem{lemma}{Lemma}[section]
\newtheorem{example}{Example}%[section]

\newtheorem{remark}{Remark}
\newenvironment{proof}{\noindent{\bf Proof:}}{\hfill\rule{2mm}{2mm}\\}
\def\titlerunning{A New Modal Framework for Epistemic Logic}
\def\authorrunning{Wang} 
 
\author{Yanjing Wang\institute{Department of Philosophy, Peking University, China}
\email{y.wang@pku.edu.cn}}

\title{A New Modal Framework for Epistemic Logic\thanks{The author thanks Pavel Naumov for insightful initial discussions on the language and semantics of this framework, and thanks Timothy Williamson for pointing out the categorization of know-wh by \textit{mention-all} and \textit{mention-some}. The author is also grateful to Yanjun Li and Jixin Liu for carefully reading an earlier draft of the paper.}}

\begin{document}

\maketitle

\begin{abstract}
Recent years witnessed a growing interest in non-standard epistemic logics of knowing whether, knowing how, knowing what, knowing why and so on. The new epistemic modalities introduced in those logics all share, in their semantics, the general schema of $\exists x \Box \phi$, e.g., knowing how to achieve $\phi$ roughly means that there exists a way such that you know that it is a way to ensure that $\phi$. Moreover, the resulting logics are  decidable. Inspired by those particular logics, in this work, we propose a very general and powerful framework based on quantifier-free predicate language extended by a new modality $\Box^x$, which packs exactly $\exists x \Box$ together. We show that the resulting language, though much more expressive, shares many good properties of the basic propositional modal logic over arbitrary models, such as finite-tree-model property and van Benthem-like characterization w.r.t.\ first-order modal logic. We axiomatize the logic over S5 frames with intuitive axioms to capture the interaction between $\Box^x$ and know-that operator in an epistemic setting.
\end{abstract} 

\section{Introduction}
%\noteYW{Also express the content of knowledge, parallel theories and modal logics. normal form, quantifier elimination, constants cannot be coded easily. misleading ··lies‘’, knowing the meaning of words. Varying domain...$\Box\phi$ is no longer expressible... since $\exists x$ may be empty.}

Standard epistemic logic studies valid reasoning patterns about \textit{knowing that}. However, in natural language, knowledge is also expressed by \textit{knowing whether}, \textit{knowing how}, \textit{knowing what}, \textit{knowing why} and so on. Recent years witnessed a growing interest in the non-standard epistemic logics of such expressions (cf. e.g., \cite{PubPlazanew,WF13,WF14,GW16,FWvD14,FWvD15,
Wang15:lori,Wang17,LiWang17,WangBKT,GvEW16,
NaumovT17,LauWang,Baltag16} and the survey \cite{WangBKT}).\footnote{This is not meant to be an exhaustive list, e.g., see also \cite{JamrogaH04,HerzigT06,JamrogaA07} about logics of knowing how in the setting of ATL.} In this line of work, various new modalities of know-wh are introduced,\footnote{Know-wh means verb \textit{know} followed by a wh-word.} all of which share the general \textit{de re} schema $\exists x \K \phi(x)$ in their semantics, e.g, ``knowing how to achieve $\phi$'' roughly means that there \textit{exists} a way such that you \textit{know that} it is a way to ensure that $\phi$ \cite{Wang17}; ``knowing why $\phi$'' means that there \textit{exists} an explanation such that you \textit{know that} it is an explanation to the fact $\phi$ \cite{Xu16}. Actually, in the early days of epistemic logic, Hintikka already used such formulations to handle \textit{knowing who} \cite{Hintikka:kab} (cf.\ the survey \cite{WangBKT} for a detailed discussion on Hintikka's early contributions). Such interpretations are  grounded also in philosophy and linguistics (cf. e.g., \cite{stanley2001knowing,stanley2011know}). 

Such a semantic schema is based on the so-called \textit{mention-some} interpretation to the wh-questions embedded in those knowledge expressions \cite{hamblin1973questions}. There is also a \textit{mention-all} interpretation \cite{SG82}, which makes sense in many other situations, e.g., ``knowing who came to the party''  means, under an exhaustive reading, that for \textit{each} relevant person, you know whether (s)he came to the party or not, which can be summarized as $\forall x (\K\phi(x)\lor \K\neg \phi(x))$. There are degenerated cases when the two interpretations coincide, e.g., ``knowing [what] the value of $c$ [is]'' means, under the interpretation of mention-some, that there \textit{exists} a value such that you \textit{know that} it is the value of $c$, which is equivalent to the mention-all interpretation: for any value, you know whether it is the value of $c$, given there is \textit{one and only one} real value of $c$.  

Given the experiences in dealing with those particular cases of know-wh, it is the time to lay out a general background  framework for the shared ``logical core'' of those  logics.\footnote{Note that this does not imply that we can simply use a single framework to cover all of those particular cases, since the details of the semantics in each setting matter a lot in deciding the characteristic axioms and rules in each different setting. Moreover, for example, in the setting of \textit{knowing how} logics, we need to quantify over second-order objects (plans or strategies). } This paper is the initial step towards this purpose by extending the  predicate language with the \textit{mention-some} operator $\Box^x$, which is essentially a package of $\exists x \Box$. With the variable in place, we can say much more than those existing logics of know-wh, e.g., ``I know a theorem of which I do not know any  proof'' $\KS \neg \K^y Prove(y, x)$, or, in a multi-agent setting, ``$i$ knows a country which $j$ knows its capital'': $\Box^x_i\Box^y_j Capital(y, x)$. Actually, when $x$ does not appear in $\phi$, $\Box^x\phi$ is equivalent to $\Box\phi,$ thus our language is indeed an extension of the standard modal language. Moreover, we will show that in the epistemic setting, this mention-some operator can also express mention-all. By having the predicate symbols in the language, we can also talk about the \textit{content} of knowledge, which may be useful to bridge epistemic logic and knowledge representation. We also  believe that the new modality can be interpreted meaningfully not only in the epistemic context.  

From a technical point of view, what we discovered is a well-behaved yet powerful fragment of first-order modal logic (over arbitrary models), as the following main  technical contributions of this paper demonstrate: 
\begin{itemize}
\item We propose a novel notion of bisimulation which can characterize the expressive power of our language within first-order modal logic precisely, over arbitrary models. 
\item Over arbitrary models, satisfiability of the equality-free fragment is not only decidable but also \PSPACE-complete, just as the complexity of basic propositional modal logic, demonstrated by a tableau-like method to show some strong finite-tree-model property.
\item We give a sound and complete proof system to our logic over epistemic (S5) models with the equivalence relation. However, we show that over S5 models, our logic is undecidable. 
\end{itemize}
Due to historical and technical reasons, first-order modal logic (\textbf{FOML}), in particular in the epistemic setting, has not been  thoroughly studied as its propositional brother (cf.\ e.g., \cite{G07FML,gochet2006epistemic,WangBKT}). Decidable fragments of first-order modal logic are usually obtained by restricting the occurrences of variables (particularly in the scope of $\Box$) (cf.\ e.g., \cite{HodkinsonWZ00,HodkinsonWZ02,Hodkinson02a,BelardinelliL11}). We hope our framework and techniques can pave a new way to some interesting fragments of first-order modal logic, using new modalities to pack quantifiers and standard modalities together, which also reflects the ``secret of success'' of basic propositional modal logic as a nicely balanced logic between expressivity and complexity. 

In the rest of this paper, we introduce  the language and semantics of our framework in Section \ref{sec.ls}, study its expressivity over arbitrary models in Section \ref{sec.exp}, prove the complexity of the equality-free fragment in Section \ref{sec.dec}, give the proof systems in the epistemic setting in Section \ref{sec.el}, and prove their completeness and undecidability in Section \ref{sec.comp}.

\section{Syntax and Semantics} \label{sec.ls}

Throughout this paper we assume a fixed countably infinite set of variables $\Var$, and a fixed set of predicate symbols $\Ps$. Furthermore, we assume that each predicate symbol is associated with a unique non-negative integer called its \textit{arity}. We use $\vec{x}$ to denote a finite sequence of (distinct) variables in $\Var$, in the order of a fix enumeration of $\Var$. By abusing the notation, we also view $\vec{x}$ as a set of variables in $\vec{x}$ occasionally. In this paper, for the brevity of presentation, we focus on the following unimodal language, but the results and techniques can be generalized to the polymodal language.\footnote{We leave the discussion about the extension with constants or function symbols to the full version of the paper.} 

%\noteYW{What about top and bot? Do we allow zero-ary predicate (propositional letters.)}

%===========
%	S Y N T A X
%===========
% \begin{definition}
% Let $\Phi$ be the minimal set of formulas such that
% \begin{enumerate}
% \item $P(v_1,\dots,v_n)\in\Phi$ for each predicate symbol $P$ of arity $n$ and each list of distinctive variables $(v_1,\dots,v_n)$,
% \item $\neg\phi\in \Phi$ for each formula $\phi\in\Phi$,
% \item $\phi\to\psi\in \Phi$ for all formulae $\phi,\psi\in\Phi$,
% \item $\Box^v_A\phi\in\Phi$ for each variable $v\in \mathcal{V}$, each set $A\subseteq\mathcal{A}$, and each formula $\phi\in\Phi$.
% \end{enumerate}
% \end{definition}

\begin{definition}[Language $\MLMS^{\approx}$] Given  $\Var$ and $\Ps$,   
%$$t::= x\mid c \mid f(\vec{t}) $$
$$ \phi::= x\approx y \mid P\vec{x} \mid \neg\phi \mid (\phi\land\phi)\mid \KS\phi $$
where $x, y\in \Var$, $P\in \Ps$. We call the  equality-free fragment $\MLMS$ (modal logic of \textit{mention-some}). 
\end{definition}
$\KS\phi$ can be read as \textit{knowing some $x$ such that $\phi(x)$} in the epistemic context. We use the usual abbreviations $\top, \bot, \lor, \to$, and write $\hKS$ for $\neg \KS\neg$, i.e., the dual of $\KS$. We define the \textit{free and bound occurrences} of variables as in first-order logic by viewing $\Box^x$ as a quantifier. We call $x$ is a \textit{free variable} in $\phi$ ($x\in FV(\phi)$), if there is a free occurrence of $x$ in $\phi$. We write $\phi(\vec{x})$ if all the free variables in $\phi$ are included in $\vec{x}$. Given an $\MLMSEQ$ formula $\phi$ and $x, y\in \Var$, we write $\phi[y\slash x]$ for the formula obtained by replacing every free occurrence of $x$ by $y$. To simplify the discussion, we do not include constant symbols and function symbols in the language and leave them to a future occasion.\footnote{Note that constant and functions can be coded using $\Ps$ and $\approx$ in full $\FOML$ (cf.\ e.g., \cite{Cresswell96}). However, our language is a fragment of $\FOML$.}. 

As for the semantics, to be general enough, following \cite{vB10,FOdef16} we use the first-order Kripke model with an \textit{increasing domain}, and flesh out the intuitive idea of mention-some discussed in the introduction. 
\begin{definition}
An (increasing domain) model $\M$ for $\MLMSEQ$ is a tuple $\lr{W, D, \delta, R, \rho}$ where: 
\begin{itemize}
\item[$W$] is a non-empty set.
\item[$D$] is a non-empty set.
\item[$R$]$\subseteq W\times W$ is a binary relation over $W$.
\item[$\delta$]$:W\to 2^D$ assigns to each $w\in W$ a \textit{non-empty} local domain s.t. $wRv$ implies $\delta(w)\subseteq \delta(v)$ for any $w,v\in W$. We also write $D_w$ for $\delta(w)$.
\item[$\rho$]$:\Ps\times W\to \bigcup_{n\in \omega}2^{D^n}$ such that $\rho$ assigns each $n$-ary predicate on each world an $n$-ary relation on $D$.\footnote{Following \cite{Cresswell96,G07FML}, we do not require that the interpretation of $P$ at a world is based on the local domain. Actually, as we will see later in Corollary \ref{co.sat}, this seemingly `counterintuitive' generalization does not affect the satisfiability or validity: each satisfiable formula $\phi$ is satisfiable in a model where $\rho(P,w)$ is based on objects in $D_w$ (cf.\ also \cite{Cresswell96}). \label{ft.inv}}  
\end{itemize} 
A \textit{constant domain} model is a model such that $D_w=D$ for any $w\in W$. A \textit{finite model} is a model with both a finite $W$ and a finite $D$. Given a model $\M$, we denote its components as $W^\M, D^\M, \delta^\M$, $R^\M$, and $\rho^\M$. To interpret free variables, we also need a variable assignment $\sigma: \Var\to D$.
%\footnote{As in Footnote \ref{ft.inv}, whether requiring $\sigma(x)$ in $D_w$ when evaluating formulas at $w$, does not affect the satisfiability of $\MLMSEQ$-formulas (also cf. \cite{Cresswell96}).} 
The formulas are interpreted on models with variable assignments. 
%\scalebox{0.9}{
$$\begin{array}{|rcl|}
%\M,  \sigma\vDash \top&\Leftrightarrow & \textrm{ always} \\ \pause
\hline
%\M, w, \sigma\vDash t\approx t' &\Leftrightarrow & \sigma(t)=\sigma(t') \\ 
\M, w, \sigma\vDash x\approx y &\Leftrightarrow & \sigma(x)=\sigma(y) \\ 
\M, w, \sigma\vDash P(x_1\cdots x_n) &\Leftrightarrow & (\sigma(x_1), \cdots, \sigma(x_n))\in \rho(P,w)  \\ 
\M, w, \sigma\vDash \neg\phi &\Leftrightarrow&   \M, w, \sigma\nvDash \phi \\ 
\M, w, \sigma\vDash (\phi\land \psi) &\Leftrightarrow&  \M, w, \sigma\vDash \phi \text{ and } \M, w, \sigma\vDash \psi \\ 
\M, w, \sigma\vDash \KS \phi &\Leftrightarrow& \text{there exists an $a\in \delta(w)$ such that }\\
&&\M, v, \sigma[x\mapsto a]\vDash\phi \text{ for all $v$ s.t.\ $wRv$}\\
%\M, \sigma\vDash \Diamond \phi &\Leftrightarrow& \text{\red{存在}$v$, } wRv \text{ 且 } \M,v\vDash\phi\\
\hline
\end{array}$$
%}

\noindent where $\sigma[x\mapsto a]$ denotes another assignment just like $\sigma$ except mapping $x$ to $a$.
%$\sigma[x\mapsto a]$就是改变$\sigma$对$x$的解释为$a$, 其他不变. 
\end{definition}
As an intuitive definition, following \cite{Cresswell96}, we say $\phi$ is \textit{valid}, if $\phi$ is true on any $\M, w$ w.r.t.\ any $\sigma$ such that $\sigma(x)\in \delta^\M(w)$ for all $x\in \Var$. Correspondingly, $\phi$ is \textit{satisfiable} if $\neg \phi$ is not valid, i.e., $\M,w,\sigma\vDash\phi$ for some $\M, w$ and $\sigma$ such that $\sigma(x)\in \delta^\M(w)$ for all $x\in \Var$.

It is not hard to see that if $\sigma(x)=\sigma'(x)$ for all the free $x$ in $\phi$, $\M, w, \sigma \vDash \phi\iff \M,w, \sigma'\vDash\phi$. In this light, we  write $\M, w\vDash \phi [\vec{a}]$ to denote $\M,w,\sigma\vDash \phi(\vec{x})$ for any $\sigma$ such that $\sigma$ assigns free variables $\vec{x}$ of $\phi$ the corresponding objects in $\vec{a}$ given $|\vec{x}|=|\vec{a}|$, where $|\!\cdot\!|$ denotes the length.  
%In the rest of the paper, we also use $\MLMSEQ$ to denote the valid formulas of $\MLMSEQ$ over increasing domain models, if the context is clear.  

For comparison, the standard semantics for $\K$ is defined as:
$$\begin{array}{|rcl|}
\hline
\M, w, \sigma\vDash \K \phi &\Leftrightarrow& \text{for all $v$ such that $wRv$ } \M, v, \sigma \vDash\phi\\
\hline
\end{array}
$$
Truth conditions of $\KS\phi$ and $\hKS \phi$ can then be defined  using $\Box, \Diamond$:
$$\begin{array}{|rcl|}
\hline
\M, w, \sigma\vDash \KS \phi &\Leftrightarrow& \text{there exists an $a\in \delta(w)$ such that }\M, w, \sigma[x\mapsto a]\vDash\Box\phi\\
\M, w, \sigma\vDash \hKS \phi &\Leftrightarrow& \text{for all $a\in \delta(w)$, }\M, w, \sigma[x\mapsto a]\vDash\Diamond\phi\\
\hline
\end{array}
$$
It is now clear that $\K\phi$ is equivalent to $\KS\phi$ where  $x\not \in FV(\phi).$ Therefore $\MLMSEQ$ can be viewed as an extension of the basic propositional modal language. Therefore, in the context of $\MLMSEQ$, $\Box\phi$ can be viewed as an abbreviation. It also becomes evident that $\KS\phi$ and $\hKS\phi$ are essentially $\exists x \K\phi$ and $\forall x\Diamond \phi$ in \FOML\ respectively. In the following we study the expressivity of $\MLMSEQ$ in comparison with \FOML.

%===========
%	S E M A N T I C S
%===========
% \begin{definition}
% Model is $(D,\{\sim\}_{a\in \mathcal{A}},\pi)$, where 
% \begin{enumerate}
% \item ```domain" $D$ is an arbitrary nonempty set,
% \item $\sim_a$ is an ``indistinguishability'' equivalence relation on $D$ for each $a\in\mathcal{A}$,
% \item $\pi$ is a function on predicate symbols such that $\pi(P)\subseteq D^n$, where $n$ is the arity of the predicate symbol $P$.  
% \end{enumerate}
% \end{definition}

% \noteYW{Don't we need possible worlds any more?}

%===========
%	S A T I S F I A B I L I T Y
%===========
% \begin{definition}
% For each function $\sigma: \mathcal{V}\to D$ and each formula $\phi\in\Phi$, satisfiability relation $\sigma\vDash\phi$ is defined as follows
% \begin{enumerate}
% \item $\sigma\vDash P(v_1,\dots,v_n)$ if $\<\sigma(v_1),\sigma(v_2),\dots,\sigma(v_n)\>\in \pi(P)$,
% \item $\sigma\vDash\neg\phi$ if $\sigma\nvDash\phi$,
% \item $\sigma \vDash \phi\to\psi$ if $\sigma \nvDash \phi$ or $\sigma \vDash \psi$,
% \item $\sigma\vDash \MA^v_A\phi$ if for each $d_1,d_2\in D$ such that $d_1\sim_A d_2$, statements $\sigma[v\mapsto d_1]\vDash\phi$ and $\sigma[v\mapsto d_2]\vDash\phi$ are either both true or both false.
% \item $\sigma\vDash \MS^v_A\phi$ if there is $d\in D$ such that for each $d'\in D$ if $d\sim_A d'$, then $\sigma[v\mapsto d']\vDash\phi$.
% \end{enumerate}
% \end{definition}

\section{Expressivity}\label{sec.exp}
 Note that our semantics for $\KS$ has the $\exists\forall$ pattern which is similar to the neighbourhood semantics for modal logic \cite{Cresswell96}. Moreover, $\MLMSEQ$ can be viewed as a fragment of the corresponding \FOML. Indeed, inspired by the \textit{world-object bisimulaion} for $\FOML$ \cite{vB10,FOdef16} and bisimulation for monotonic neighbourhood modal logic \cite{Pauly2000,MML03}, we propose a novel notion of bisimulation for $\MLMSEQ$. Before the formal definition, given a model $\M$, let $D_\M^*$ be the set of (possibly empty) finite sequence of objects in $D^\M$. A \textit{partial isomorphism} between $\vec{a}\in {D^*_\M}$ and $\vec{b}\in {D^*_\N}$ such that $|a|=|b|$ is an isomorphism mapping $a_i$ to $b_i$ w.r.t.\ relevant interpretations of predicates (cf. e.g, \cite{vB10}). It is \textit{partial} since it is not about all the objects in $D^\M$ and $D^\N$. 
\begin{definition}[$\ef$-Bisimulation]  Given two models $\M$ and $\N$, the relation $Z\subseteq (W^\M\times D_\M^*)\times (W^\N\times D_\N^*)$ is call an \textit{$\exists\K$-bisimulation}, if for every $((w,\vec{a}), (v, \vec{b}))\in Z$ such that $|\vec{a}|=|\vec{b}|$ the following holds (for brevity, comma in $(w,\vec{a})$ is omitted):
\begin{itemize}
\item[PISO] $\vec{a}$ and $\vec{b}$ form a partial isomorphism w.r.t.\ identity and interpretations of predicates at $w$ and $v$ respectively. 
\item[$\ef$Zig] For any $c\in D^\M_w$, there is a $d\in D^\N_v$ such that for any $v'\in W^\N$ if $vR v'$ then there exists $w'$ in $W^\M$ such that $wR w'$ and $w'\vec{a}c Z v'\vec{b}d$. 
\item[$\ef$Zag] For any $d\in D^\N_v$, there is a $c\in D^\M_w$ such that for any $w'\in W^\M$ if $wR w'$ then there exists $v'$ in $W^\N$ such that $vR v'$ and $w'\vec{a}c Z v'\vec{b}d$.
%\footnote{Pay attention to the alternation of models in the condition.} 
\end{itemize}
We say $\M, w\vec{a}$ and $\N,v\vec{b}$ are $\ef$-\textit{bisimilar} ($\M, w\vec{a}\bis_{\ef} \N,v\vec{b})$ if $|a|=|b|$ and there is an $\exists\K$-bisimulation linking $w\vec{a}$ and $v\vec{b}$. In particular, we say $\M,w$ and $\N,v$ are $\ef$-bisimilar if $\M, w\bis_{\ef} \N,v$, i.e., when $|\vec{a}|=|\vec{b}|=0$.  
\end{definition}
It is not hard to show that $\bis_{\ef}$ is indeed an equivalence relation. 
%\noteYW{Need to spell out.}

Note that our bisimulation notion is much weaker than isomorphism, in particular the domains of the two bisimilar models do not necessarily have the same cardinality.
\begin{example}\label{ex,msbis}
Consider the constant domain models $\M, \N$:
$$
\xymatrix@R-20pt{
\M:&\underline{w} \ar[r]\ar[dr]& v: Pa  & \N: & \underline{s} \ar[r]\ar[dr]& t: Pc\\
   &                    &u: Pb    &    &         & r
}
$$
where $D^\M=\{a,b\}$, $D^\N=\{c\}$, $\rho^\M(P,w)=\rho^\N(P,s)=\rho^\N(P,r)=\emptyset$, $\rho^\M(P,v)=\{a\}$, $\rho^\M(P,u)=\{b\}$, $\rho^\N(P,t)=\{c\}$. Suppose $P$ is the only predicate, we can show that $\M,w\bis_{\ef}\N,s$ by an $\exists\K$-bisimulation $Z$ (pay attention to the switch of the two models in the second half of the definitions of $\ef$Zig and $\ef$Zag): 
$$\{(w,s), (va, tc), (ub, tc), (vb, rc), (ua, rc) \}$$
Also note that $\ef$Zig and $\ef$Zag hold trivially for $w\vec{a}$ and $v\vec{b}$ if ${w}$ and $v$ \textit{do not} have any successor, based on the fact that local domains are non-empty by definition.  
\end{example}

We write $\M, w\vec{a}\equiv_{\MLMSEQ} \N, v\vec{b}$ if $|\vec{a}|=|\vec{b}|$ and for all any $\MLMSEQ$ formula $\phi(\vec{x})$ such that $|\vec{x}|=|\vec{a}|$: $\M, w\vDash \phi [\vec{a}] \iff \N, v\vDash \phi[\vec{b}].$ We can show that $\MLMSEQ$ is invariant under $\ef$-bisimilarity.  
\begin{theorem}\label{thm.inv}
$\M, w\vec{a}\bis_{\ef} \N,v\vec{b}$ then $\M, w\vec{a}\equiv_{\MLMSEQ} \N, v\vec{b}$. 
\end{theorem}
\begin{proof}
It suffices to prove that if $Z$ is an $\exists\K$-bisimulation linking $w\vec{a}$ and $ v\vec{b}$ such that $|\vec{a}|=|\vec{b}|$,  then $\M, w\vDash \phi [\vec{a}] \iff \N, v\vDash \phi[\vec{b}]$.
%for any $w\in W^\M, v\in W^\N, \vec{a}\subseteq D_M, \vec{b}\subseteq D^\N$ . 
Given a bisimulation $Z$, we do induction on the structure of $\phi$. Supposing $\phi$ is an atomic formula, due to PISO, $\M,w\vDash \phi [\vec{a}]$ iff $\N,v\vDash\phi [\vec{b}]$ for any atomic formula $\phi (\vec{x})$. 

Suppose $\phi$ is in the shape of $\KS\psi(\vec{x})$ and  $\M,w\vDash \phi [\vec{a}]$. By the semantics there is a $c\in D^\M_w$ such that $\M,w'\vDash \psi [\vec{a}c]$ for all the $w'$ such that $wRw'.$ According to $\ef Zig$, there exists a $d\in D^\N_v$ such that the second half of the condition holds. We claim that  $\N,v'\vDash \psi [\vec{b}d]$ for any $v'$ such that $vRv'.$ Suppose not, then there is a $v'$ such that $\N,v'\nvDash \psi [\vec{b}d]$ and $vRv'.$ Then according to $\ef Zig$, there is $w'$ such that $wRw'$ and  $w'\vec{a}c Zv'\vec{b}d$. By IH, $\M,w'\nvDash \psi [\vec{a}c]$. Contradiction. 
\end{proof}

\begin{corollary}
$\M, w\bis_{\ef} \N,v$ then for any closed $\MLMSEQ$-formula $\phi$: 
$\M, w\vDash \phi \iff \N, v\vDash \phi.$
\end{corollary}
By the above invariance results, we can show that many natural  combinations of quantifiers and modalities are not expressible in $\MLMSEQ$. 
\begin{proposition}
$\K \exists x Px$, $\exists x \hK Px$ and $\hK\exists x Px$ are not expressible in $\MLMSEQ$. 
\end{proposition}
\begin{proof} In this proof, we again consider constant domain models.  For $\K \exists x Px$, consider the bisimilar models in Example \ref{ex,msbis}. $\Box\exists x Px$ holds on $\M,w$ but not on $\N,s$, thus it is not expressible in $\MLMSEQ$. 
% $$
% \xymatrix@R-20pt{
% \M:&\underline{w} \ar[r]\ar[dr]& v: Pa  & \N: & \underline{s} \ar[r]\ar[dr]& t: Pc\\
%    &                    &u: Pb    &    &         & r
% }
% $$
% where $D^\M=\{a,b\}$, $D^\N=\{c\}$, $P^\M(v)=\{a\}$, $P^\M(u)=\{b\}$, $P^\N(s)=\{c\}$, $P^\M(w)=P^\N(s)=P^\N(r)=\emptyset$. We can show that $\M,w\bis_{\ef}\N,v$ by a $\exists\K$-bisimulation $Z$: 
% $$\{(w,s), (va, tc), (vb, rc), (ua, rc) \}$$
For $\exists x \hK Px$ and $\hK\exists x Px$, consider:  
$$
\xymatrix@R-20pt{
\M:&\underline{w} \ar[r]\ar[dr]& v: Pa  & \N: & \underline{s} \ar[r]& t\\
   &                    &u    &    &         & 
}
$$
where $D^\M=\{a,b\}$, $D^\N=\{c\}$ as before, $\rho^\M(P,w)=\rho^\M(P,u)=\rho^\N(P,t)=\rho^\N(P,s)=\emptyset$, $\rho^\M(P,v)=\{a\}$. Clearly, $\exists x\hK Px$ and $\hK\exists x Px$ are true at $\M, w$ but false at $\N, s$. However, we can show that $\M,w\bis_{\ef}\N,s$ by an $\exists\K$-bisimulation $Z$: 
$$\{(w,s), (ua, tc), (vb, tc), (ub, tc)\}.$$
% \xymatrix{
% \M:&\underline{w} \ar@(dl, ul)\ar@{-}[r]& v \ar@(dr, ur) & \N: & s \ar@(dl, ul)\ar@{-}[r]& t \ar@(dr, ur)
% }
% where $D^\M=\{a,b\}$, $D^\N=\{c,d\}$, $P^\M(w)=\{a\}$, $P^\M(v)=\{b\}$, $P^\N(s)=\{c\}$, $P^\N(t)=\emptyset$. We can show that $\M,w\bis_{\ef}\N,v$ by a $\exists\K$-bisimulation $Z$ extending the following set: 
% $$\{(w,s), (wa, sc), (va, tc), (wb, sd), (vb, td), (va, sd)\}$$
\end{proof}
To precisely characterize $\MLMSEQ$ within the corresponding $\FOML$, we need a notion of saturation. In the following, we write $\Gamma(\vec{x})$ if all the free variables in the set of $\MLMSEQ$-formulas $\Gamma$ are included in  $\vec{x}$. Inspired by \cite{FOdef16}, we can generalize the concept of m-saturation for propositional modal logic (cf.\ \cite{mlbook}) as follows:  
\begin{definition}
A model $\M$ is said to be \textit{$\ef$-saturated}, if for any $w\in W^\M$, and any finite sequence $\vec{a}\in {D^*_\M}$, the following two conditions are satisfied: 
\begin{itemize}
\item[$\exists\Box$-type] If for each finite subset $\Delta$ of a set $\Gamma (\vec{y}x)$ where $|\vec{y}|=|\vec{a}|$, $\M, w\vDash\KS \bigwedge\! \Delta [\vec{a}]$,\footnote{Here we require $\vec{y}$ are assigned $\vec{a}$ correspondingly. This is only to avoid introducing extended language with constants for $\vec{a}$ as in standard model theory, since we did not define interpretation for constant symbols. Similarly below.} then there is a $c\in D^\M_w$ such that $\M, w\vDash \Box \phi[\vec{a}c]$ for all $\phi\in \Gamma$, where $x$ is assigned $c$.  
\item[$\Diamond$-type] If for each finite subset $\Delta$ of   $\Gamma(\vec{x})$ such that $|\vec{x}|=|\vec{a}|$, $\M, w\vDash\Diamond \bigwedge\!\Delta [\vec{a}]$, then $\M, w\vDash\Diamond \phi[\vec{a}]$ for each $\phi\in\Gamma$.
\end{itemize}
\end{definition}
Note that in the above definition, to simplify the presentation, we use $\Box, \Diamond$ which are expressible in $\MLMSEQ$. $\Diamond$-type condition is essentially the m-saturation adapted with variable assignments.  Recall that a finite model has a finite domain and finitely many worlds. It can be verified that:\footnote{Due to limited space, we omit some proofs and leave them to the extended version of this conference paper.} \begin{proposition}
Every finite model is $\ef$-saturated. 
\end{proposition}
%provide a proof in the full version.
%\begin{proof}
%\end{proof}

We can obtain the Hennessy-Milner-type theorem (cf.\ \cite{mlbook}) to establish the equality between $\ef$-bisimilarity and $\MLMSEQ$-equivalence.  
\begin{theorem}\label{thm.hm}
For $\ef$-saturated models $\M,\N$ and $|\vec{a}|=|\vec{b}|$: $$\M, w\vec{a}\bis_{\ef} \N,v\vec{b} \iff \M,w \vec{a}\equiv_{\MLMSEQ} \N,v\vec{b}$$ 
\end{theorem}
\begin{proof}
Due to Theorem~\ref{thm.inv}, we only need to show the right-to-left direction. We define $Z=\{(w\vec{a}, v\vec{b})\mid w\in W^\M, v\in W^\N, \vec{a}\in D_\M^*, \vec{b}\in D_\N^*, |\vec{a}|=|\vec{b}|, \M, w\vec{a}\equiv_{\MLMSEQ} \N, v\vec{b}\}$. We need to show that $Z$ is an $\exists\K$-bisimulation. PISO is straightforward as $\vec{a}$ and $\vec{b}$ are finite and partial isomorphism between them can be expressed by atomic formulas. We only show $\ef Zig$ since $\ef Zag$ is similar. Assuming $w\vec{a}Z v\vec{b}$ and there is a $c\in D^\M_w$. Let $\Gamma=\{\phi(\vec{y}x)\mid w\vDash \Box \phi [\vec{a}c], |\vec{y}|=|\vec{a}|\}$. Now for any finite set $\Delta\subseteq \Gamma$ we have $w\vDash \KS\bigwedge\! \Delta[\vec{a}]$. Since  $w\vec{a}\equiv_{\MLMSEQ}v\vec{b}$, $v\vDash \KS\bigwedge\! \Delta[\vec{b}]$. Now by $\exists\Box$-type condition, we know there is a $d\in D^\N_v$ such that $v\vDash \Box\phi [\vec{b}d]$ for all $\phi\in\Gamma$ ($\star$). Now take an arbitrary $v'$ such that $vR v'$, we show there is a $w'$ such that $wRw'$ and $w'\vec{a}c\equiv_{\MLMSEQ}v'\vec{b}d$. Let $\Sigma=\{\phi(\vec{x})\mid \N,v'\vDash \phi [\vec{b}d], |\vec{x}|=|\vec{b}d|\}$. It is clear that for each finite set $\Delta\subseteq \Sigma$, $v\vDash\Diamond \bigwedge\! \Delta[\vec{b}d]$, namely $v\vDash\neg \Box \neg \bigwedge\! \Delta[\vec{b}d]$. We can show $w\vDash \neg \Box \neg \bigwedge\! \Delta[\vec{a}c]$. Suppose not, then $w\vDash \Box \neg \bigwedge\! \Delta[\vec{a}c]$, thus $\neg \bigwedge\! \Delta\in\Gamma$. By ($\star$), $v\vDash \Box\neg \bigwedge\! \Delta[\vec{b}d]$, which is in contradiction with $v\vDash\neg \Box \neg \bigwedge\! \Delta[\vec{b}d]$. Now it is clear that $w\vDash \Diamond \bigwedge\! \Delta[\vec{a}c]$ for each finite $\Delta\subseteq \Sigma$. By the $\Diamond$-type condition of $\ef$-saturation, there is a successor $w'$ of $w$ such that $w'\vDash \Sigma[\vec{a}c]$. Therefore $w'\vec{a}c\equiv_{\MLMSEQ} v'\vec{b}d,$ namely $(w'\vec{a}c, v'\vec{b}d)\in Z$. This completes the proof for $\ef Zig$.
\end{proof}
\begin{corollary}
 For finite models: $\M, w\vec{a}\bis_{\ef} \N,v\vec{b} \iff \M,w \vec{a}\equiv_{\MLMSEQ} \N,v\vec{b}$. For $\ef$-saturated models: $\M, w\bis_{\ef} \N,v \iff \M, w \text{ and }\N,v$ satisfies the same closed $\MLMSEQ$-formulas (sentences).
\end{corollary} 
Now we can characterize $\MLMSEQ$ in the corresponding first-order language. 
\begin{definition} Given $\Var, \Ps$ as before, 
the corresponding $\FOML$ language is defined as follows: 
$$ \phi::= x\approx y \mid P\vec{x} \mid \neg\phi \mid (\phi\land\phi)\mid  \forall x\phi \mid \Box\phi $$
The corresponding two-sorted first-order language $\SFOL$ is: 
$$ \phi::= x\approx y \mid Q u\vec{x} \mid Ruv \mid Eux \mid \neg\phi \mid (\phi\land\phi)\mid  \forall x\phi \mid \forall u\phi $$
where $x, y\in \Var$ and $u,v\in \War$ which is a collection of \textit{world variables} disjoint from $\Var$; $Q\in \Qs$ and $\Qs$ is the smallest collection of predicate symbols such that for each $n$-ary $P\in \Ps$ there is a unique $Q_P\in \Qs$ with the arity  $n+1$. $E$ is a new predicate symbol to say which  object exists on which world.
\end{definition}
It is trivial to define a translation $r$ from $\MLMSEQ$ to the corresponding $\FOML$ by setting $r(\KS\phi)=\exists x \Box r(\phi)$. It is clear that this translation is truth preserving. To translate $\MLMSEQ$ into the formulas in the corresponding $\SFOL$ (with $u,v$ as the only world variables), we define the following $t_u$ inspired by \cite{vB10}:    
\begin{definition}[Translation from $\MLMSEQ$ to $\SFOL$]
$$\begin{array}{l}
t_u(x\approx y)= x\approx y \qquad t_u(P\vec{x})= Q_P(u, \vec{x}) \\
t_u(\neg \psi)=\neg t_u(\psi) \qquad t_u(\phi\land\psi)=t_u(\phi)\land t_u(\psi).\\
t_u(\KS\psi)=\exists x (Eux \land \forall v (Ruv \to t_{v}(\psi)) \\
\end{array}
$$
$t_v$ is defined symmetrically, by swapping $u$ and $v$.
\end{definition}
In this way, every $\MLMSEQ$-formula is translated into a two-world-variable formula of $\SFOL$ with one free variable, similar to the standard translation of  basic modal language to first-order language. 

We can also view a model for $\MLMSEQ$ as a model for $\SFOL$ by turning $\rho(w, P)$ into the interpretation of the corresponding $Q_P$ in the most natural way. It is not hard to show: 
\begin{proposition} For any $\MLMSEQ$ formula $\phi$: 
$$\M,w, \sigma\vDash \phi \iff \M, \sigma' \Vdash t_u(\phi)$$
\noindent where $\sigma'$ is $\sigma$ extended by $u\mapsto w$, $\Vdash$ is the standard semantics for $\SFOL$ (cf. e.g. \cite{G07FML}).  
\end{proposition}
It follows from the compactness of $\SFOL$ that:\footnote{\label{ft.red}\SFOL\ inherits the compactness from first-order logic since there is a truth preserving translation from \SFOL\ to \FOL\ by using two new unary predicates $S_1$, $S_2$ for the two sorts, together with a single unrestricted quantifiers to mimic the two-sorted quantifiers. Note that we can also express the constraints on the sorts in $\FOL$ by $\theta(S_1,S_2)=\forall x ((S_1 x\lor S_2x) \land \neg (S_1x\land S_2x))$.} 
\begin{proposition}\label{prop.compact}
$\MLMSEQ$ is compact.
\end{proposition}
%\begin{proof}
%We can translate any set of $\MLMSEQ$-formulas $\Gamma$ into a set of $\SFOL$-formulas $t_u(\Gamma)$ by the translation $t_u$. The extra constraints on the increasing domain and non-emptiness can be expressed in $\SFOL$ as $\chi_1=\forall u\forall v \forall x(Eux\land Ruv\to Evx)$ and $\chi_2=\forall u \exists x Eux$ respectively. 
%%The fact that each world has a fixed interpretation for a particular predicate $P$ is expressed by $\forall u \forall \vec{x}(P )$
%Therefore if $\Gamma$ is finitely satisfiable at some increasing-domain pointed model with some variable assignment, then $t_u(\Gamma)\cup\{\chi_1, \chi_2\}$ is satisfiable in a $\SFOL$ model by the compactness of $\SFOL$. Then $\Gamma$ is satisfiable in an increasing-domain model for $\MLMSEQ$.     
%\end{proof}
Based on Theorem \ref{thm.hm} and the compactness of $\MLMSEQ$, it is relatively routine to prove the van-Benthem-like characterization theorem using the strategy in \cite{vB10,mlbook}, which makes use of $\omega$-saturation in first-order model theory.  We omit the proof due to limited space. 
\begin{theorem}
A $\SFOL$-formula $\phi(\vec{x}u)$ is equivalent to an $\MLMSEQ$-formula (over $\MLMSEQ$-models) iff it is invariant under $\ef$-bisimilarity. 
\end{theorem}

Since any $\FOML$ formula can also be viewed as a $\SFOL$ formula $\phi(\vec{x}u)$, via a natural translation like $t_u$ (cf. \cite{G07FML,FOdef16}), it follows that: 
\begin{corollary}
A $\FOML$-formula is equivalent to an $\MLMSEQ$-formula iff it is invariant under $\ef$-bisimilarity.  
\end{corollary}

\section{Satisfiability of $\MLMS$} \label{sec.dec}

Note that by $t_u$ we can translate an $\MLMSEQ$ formula to an equivalent $\SFOL$ formula, and eventually to a $\FOL$ formula with one free world variable $u$ (see Footnote \ref{ft.red}). For example $\KS Px$ becomes $$\phi(u)=\exists x (S_1x \land S_2u \land Eux\land \forall v ((S_2v \land Ruv)\to Q_Pvx))$$ in the corresponding $\FOL$ language. Note that since we only consider increasing domain models, $\phi$ is equivalent to 
$$\phi(u)=\exists x (S_1x \land S_2u \land Eux\land \forall v ((S_2v \land Ruv \land Evx) \to Q_Pvx))$$
which is in the (loosely) guarded fragment of first-order logic known to be decidable \cite{GF98}. However, it does not directly imply that our logic is decidable, since, for example, to capture the increasing domain model we do need the following first-order constraint: 
$$ \chi=\forall u\forall v \forall x((S_2u\land S_2v\land S_1x\land Eux\land Ru v)\to Evx) $$
$\chi$ can be viewed as some form of transitivity, and it is not in the guarded fragment, since $v$ and $x$ are free in the consequent but they do not co-exist in any of the atomic guards in the antecedent. On the other hand, such a transitivity-like constrained predicate $E$ only appears in the guards of the translations of $\MLMSEQ$-formulas, which may be handled by the techniques in \cite{TansGF} for the guarded fragment with transitive guards known to be decidable. We leave the detailed discussion connecting $\MLMSEQ$ to guarded fragments to a future occasion. 
\medskip

In this section, we give an intuitive tableau-like method to decide whether an equality-free \MLMS-formula is satisfiable, inspired by the  satisfiability games for various temporal logics \cite{Gore99,Lange}.\footnote{See \cite{PriestIfIs,FittingM1998} for tableau methods for first-order modal logic in general.} 

For the ease of presentation, we consider the following equivalent language of $\MLMS$ in \textit{positive normal form (PNF)}, where negations only appear with atomic formulas:\footnote{PNF is often used in automata-theoretical methods to satisfiability problems of logics in computer science (cf. e.g., \cite{MuvsMSOLNCS2500,Lange}).}
$$ \phi::=  P\vec{x} \mid \neg P\vec{x} \mid (\phi\land \phi) \mid (\phi\lor\phi)\mid \KS\phi \mid \hKS\phi $$
Note that for an arbitrary $\MLMS$-formula $\phi$, we can rewrite it into an equivalent formula in PNF by the following rewriting rules and the replacement of equals: 
$$ r(\neg (\phi\land \psi))=r(\neg \phi)\lor r(\neg \psi)\quad r(\neg \KS\phi)=\hKS r(\neg \phi) \quad r(\neg \neg \phi)=r(\phi) $$
Note that $|r(\phi)|\leq 2|\phi|$. 

In the following, we will focus on the formulas that are \textit{clean} in the sense that no variable occurs both free and bound,\footnote{\label{ft.cl} In particular, if $\KS$ or $\hKS$ appears in the formula then $x$ does not have any free occurrence, even when $\KS$ and $\hKS$ do not bind any occurrence of $x$, e.g., $(\KS Py) \land Px$ is \textit{not} clean.} and no two distinct occurrences of modalities bind the same variable. As in first-order logic, it is easy to show that any $\MLMS$ formula can be relettered into an equivalent clean formula with the same length, by renaming the bounded variables.
%We say a finite set of formulas \textit{clean} if their conjunction is clean.\footnote{Note that this is stronger than asking each formula in the set is clean.} 
We define the following tableau rules for all $\MLMS$ formaulas in PNF. 
\begin{definition}{Tableau rules}
\begin{center}
\noindent \begin{tabular}{|c|}
\hline
$\dfrac{w: \phi_1\lor\phi_2,\Gamma, \sigma}{w:\phi_1,\Gamma, \sigma \mid w:\phi_2,\Gamma, \sigma}$ \tt{($\lor$)}\qquad $\dfrac{w: \phi_1\land\phi_2,\Gamma, \sigma}{w:\phi_1,\phi_{2},\Gamma,\sigma}${\tt ($\land$)}\\
\hline
Given $n\geq 0, m\geq 1$:\\
$\dfrac{w:\KSxf\phi_1, \dots, \KSxl\phi_n,\hKSyf\psi_1,\dots, \hKSyl\psi_m, l_1\dots l_k, \sigma}{\{(wv^y_{y_i}: \{\phi_j\mid 1\leq j\leq n\}, \psi_i[y\slash y_i], \sigma') \mid y\in Dom(\sigma'), i\in [1,m]\}} $ \tt{($\BR$)}\\
\hline 
Given $n\geq 1, k\geq 0$:\\
$\dfrac{w:\KSxf\phi_1, \dots, \KSxl\phi_n, l_1\dots l_k, \sigma}{w: l_1\dots l_k, \sigma}  $ (\END) \\
\hline
\end{tabular}\\
%}\\
\noindent where $\sigma'=\sigma\cup\{(x_j, x_j)\mid j\in [1,n] \}$ and $l_k\in lit$ (the literals).
\end{center}
\end{definition}
Those rules are used to generate tree-like structure, where the nodes are triples $(w,\Gamma,\sigma)$ in which $w$ is some name to denote a world, $\Gamma$ is a finite set of \MLMS-formulas, and $\sigma$ is a partial function from $\Var$ to $\Var$ as an assignment. Note that Rule $(\lor)$ is essentially a \textit{choice}: given the numerator, you can only select \textit{one} of the denominator to continue. On the other hand, the rule ($\BR$) is a \textit{branching} one, which generates all the nodes in the denominator set (even when $n=0$ i.e., the $\KS$-part is empty). It is a generalized version of the corresponding rule for basic modal logic (cf. \cite{Lange}): 
$$\dfrac{w: \Box\phi,\Diamond \psi, l_1\dots l_l}{v: \phi, \psi}$$
The idea is that if there is a diamond formula then we need to generate a successor while keeping the information given by $\Box$-formulas. Here the complication is to manage the variable assignment properly. Note that if there are merely $\KS$-formulas without any $\hKS$-formula, then we do not need to generate any successor, as captured by the rule $\END$. 

A \textit{tableau} starting from $(w:\Gamma, \sigma)$ is a tree where the successors of a node are generated by applying the rules,  until no rule is applicable. Recall that you can only select one of the denominators to continue when applying $(\lor)$. Below is an example of a tableau where the (partial) function $\sigma$ is represented by a set of ordered pairs: \\
%\scalebox{0.8}{
\begin{center}
\begin{forest}
%\forestset{justifications=true}
[{$w: \{\KS(Px\lor Qx)\land \hKSy \neg Qy\land \neg Pz, \{(z,z)\}$ \quad $(\land)\times 2$}
[{$w: \{\KS(Px\lor Qx), \hKSy \neg Qy, \neg Pz\} \{(z,z)\}$ \qquad $(\BR)$}
[{$wv^x_y: \{Px\lor Qx, \neg Qx\}, \{(x,x), (z,z)\}$\ \quad ($\lor$)}
[{$wv^x_y: \{Px, \neg Qx\}, \{(x,x), (z,z)\}$}]
]
[{$wv^z_y: \{Px\lor Qx, \neg Qz\}, \{(x,x), (z,z)\}$ }
[{$wv^z_y: \{Qx, \neg Qz\}, \{(x,x), (z,z)\}$}]
]
]
]
\end{forest}
\end{center}

%}

% \scalebox{0.7}{$$
% \xymatrix@R-20pt{
% &w: \{\KS(Px\lor Qx)\land \hKSy \neg Qy\land \neg Pz\}, \{(z,z)\} & \\
%  &w: \{\KS(Px\lor Qx), \hKSy \neg Qy, \neg Pz\}, \{(z,z)\} &\\
% {wv^x_y: \{Px\lor Qx, \neg Qx\}, \{(x,x), (z,z)\}} &&{wv^z_y: \{Px\lor Qx, \neg Qz\}, \{(x,x), (z,z)\}}\\
% wv^x_y: \{Px, \neg Qx\}, \{(x,x), (z,z)\} && wv^z_y: \{Qx, \neg Qz\}, \{(x,x), (z,z)\}
% }
% $$
% }
Intuitively, $(\lor)$ and $(\land)$ will `decompose' the formula until $\BR$ is applicable. It is not hard to see that all the leaf nodes are in the shape of $(w: l_1\dots l_k, \sigma)$. Moreover, any tableau starting from $(w: \Gamma, \sigma)$, where $\Gamma$ and $\sigma$ are finite, is a finite tree, both in depth and width, since we only generate simpler formulas by the rules and the domain of the assignment is always finite. A tableau is called \textit{open} if all its branches do not contain contradictions of literals at the same world, i.e. no $P\vec{x}$ and $\neg P\Vec{x}$ appear together.  

Given a tableau $\T$, we say a node $(w: \Gamma, \sigma)$ is a \textit{branching node} if it is branching due to the application of $\BR$. We call $(w: \Gamma, \sigma)$ the \textit{last node of} $w$, if it is a leaf node or a branching node. Clearly, given a $w$ appearing in a tableau $\T$, the last node of $w$ always uniquely exists, since we always only select one of the denominators for the rule $(\lor)$. We denote the last node of $w$ in a given $\T$ as $t_w$. Let $$Dom(t_w)=\left\{\begin{array}{ll} Dom(\sigma') & \text{ if $t_w$ is branching }\\ Dom(\sigma) & \text{otherwise} \end{array}\right.$$
where $\sigma$ is the assignment in $t_w$ and $\sigma'$ is the assignment defined as in $\BR$ w.r.t.\ $t_w$.
%We say $w$ is a \textit{branching world} if $t_w$ in $\T$ is branching.  
% Here are some handy observations regarding the tableaux: 
% \begin{proposition}Given a tableau $\T$:  

% \end{proposition}
% \begin{proof}
% For (1)
% \end{proof}
Actually, the open tableaux are pseudo models.  
\begin{theorem}\label{thm.tbl}For any clean $\MLMS$-formula $\theta$ in PNF, the following are equivalent: 
\begin{itemize}
\item There is an open tableau from $$(r: \{\theta\}, \sigma_r=\{(x, x)\mid x \text{ is free in } \theta \}\cup \{(z,z)\})$$ where $z\in \Var$ and it does not appear in $\theta$.
\item $\theta$ is satisfiable in an increasing domain model.  
\end{itemize}
%\vspace{-20pt}
\end{theorem}
\begin{proof}
Before the main proof, we need the following handy observations about any tableau $\T$ starting from $(r, \{\theta\}, \sigma_r)$ where $\theta$ is clean. For any node $(v: \Gamma, \sigma)$ in $\T$, we claim: 
\begin{itemize}
% \item[(1)] If $\KS \phi, \hKSy \psi\in \Gamma$ then any occurrence of $x$ in $\psi$ or $\phi$ is free. 
% \item[(2)] If $\hKSy\psi\in\Gamma$ then any occurrence of $x\in Dom(\sigma)$ in $\psi$ is free. 
%\item[(1)] If $\KS \phi, \hKSy \psi\in \Gamma$ then there is no modality binding $x$ or $y$ appearing in $\phi$ or $\psi$.  
\item[(1)] If $\KS$ ($\hKS$) occurs in $\Gamma$, then it only occurs once and there is no $\hKS$ ($\KS$) occurring in $\Gamma$.
%or $\hKS$) occurs, then it only occurs there is at most one occurrence of some modality biding $x$.\footnote{For example, if $\KS$ appears in $\Gamma$, then it only occurs once in the whole $\Gamma$, and $\hKS$ cannot appear in $\Gamma$.} 
\item[(2)] For all $x\in Dom(\sigma)$, $\KS$ and $\hKS$ do not appear in $\Gamma$, thus all the occurrences of $x$ in $\Gamma$ are free.  
\item[(3)] All the free variables $x$ in $\Gamma$ are in $Dom(\sigma)$. 
\item[(4)] If $\KS \phi, \KSy\psi \in \Gamma$ then $y$ is not a free variable in $\phi$.
\item[(5)] If $\KS \phi, \hKSy\psi \in \Gamma$ then $y$ is not a free variable in $\phi$ and $x$ is not a free variable in $\psi$. 
\item[(6)] For any $x$ in $Dom(\sigma)$, $\sigma(x)=x$.
%\item For any branching node, any $y\in Dom(\sigma')$, $y$ is not bounded in $\psi_i$ for $i\in[1, m]$, where $\sigma'$ and $\psi_i$ are as in the numerator of the rule $\BR$ applicable to this node. 
%\item If $x$ occurs freely in $t_w$ for some $w$, then $x\in Dom(\sigma)$ where $\sigma$ is the assignment of $t_w$.    
\end{itemize}
(1): We prove it by induction on the structure of $\T$ from the root: it is true for the clean formula $\theta$ by definition, and all the rules preserve this property since they never add any new occurrences of modalities.     

(2): Again, we prove it by induction from the root: at the root the claim is true by the definition of $\sigma_r$ and the cleanness of $\theta$ (cf. also Footnote \ref{ft.cl}). Moreover, all the rules preserve this property. In particular, for the rule $(\BR)$, by induction hypothesis, for any variable $x\in Dom(\sigma)$, it only occurs free in the formulas of the numerator. Thus the occurrences of $x\in Dom(\sigma)$, if any, are also free in any $\phi_j$ and $\psi_i[y\slash y_i]$ which have less modalities to bind than the numerator. Now for $x\in Dom(\sigma')\setminus Dom(\sigma)$, i.e., $\KS$ appears in the numerator, the statement also holds by Claim (1), since there is only one occurrence of any modality for such an $x$ in the numerator.

(3): Again, we can show by induction: $Dom(\sigma_r)$ has all the free variables in $\theta$, and all the rules preserve this property. For the rule $(\BR)$, we need to check the free variables in those $\phi_j$ and $\psi_i$. Note that the only possible extra free variable in $\phi_j$ but not in $\Box^{x_j}\phi_j$ is $x_j$ but it is already included in $Dom(\sigma')$. The only possible extra free variable in $\psi_i[y\slash y_i]$ but not in $\Diamond^{y_i}\psi_i$ is $y$ which is also already in $Dom(\sigma')$.  

(4): Towards contradiction, suppose $\KS \phi, \KSy\psi \in \Gamma$, but $y$ is a free variable in $\phi$. According to Claim (3) $y\in Dom(\sigma)$, then by Claim (2), $\KSy$ does not appear in $\Gamma$, which is in contradiction with $\KSy\psi\in\Gamma$.   

(5): Similarly to (4).

(6): Obvious, by definition. 

\medskip
Now we are ready to prove the main theorem. 

\textbf{From top to bottom}: Given an open tableau $\T$ from the root node $(r:\phi, \sigma_r)$, we define $\M=\lr{W, D, \delta, R, \rho}$ where: 
\begin{itemize}
\item $W=\{w\mid (w, \Gamma, \sigma) \text{ appears in $\T$ for some $\Gamma$ and $\sigma$}\}$ 
\item $w R v$ iff $v=wv'$ for some $v'$.
\item $\delta(w)=Dom(t_w)$
\item $D=\bigcup_{w\in W}\delta(w)$
%is the last node (root-down) containing $w$ (call it $\sigma_w$).
\item $\vec{x}\in \rho(w, P)$ iff the atomic formula $P\vec{x}$ appears in $t_w$. 
\end{itemize}
\noindent where $t_w$ is the last node of $w$ in $\T$, as defined before. 
 
According to the rule $(\BR)$ and the definition of $\delta$, $\M$ is indeed an increasing domain model.\footnote{Note that it does not mean $D_w$ is everywhere the same due to branching nodes in $\T$.} Since $(z,z)\in Dom(\sigma_r)$, $D_w$ is not empty for any $w\in W$. Moreover $\rho$ is well-defined due to the openness of $\T$. Note that due to Claim (3), if the atomic formula $P\vec{x}$ appears in $t_w$ then $\vec{x}\subseteq D_w$. We will show that $\M,r$ is indeed a model of $\theta$ w.r.t.\ $\sigma_r$. 
% and the fact that if $P\vec{x}$ is contained in the last node $w: \Gamma, \sigma$ of $w$ then $Ran(\sigma)\subseteq D_w=Ran(\sigma)$. 
%Also note that since $\phi$ is a closed formula, it must have one modal operator $\KS$ or $\hKS$  according to the definition of $\MLMSN$. Then by the rule $\tt{\hK}$ or $\tt{END}$ we know $D_r$ is not empty.  
% We also need an observation before proceeding further. Note that $\phi$ is a closed formula, thus all the free variables  appearing in $\T$ are created by applying rule ($\BR$), which  also makes sure that $\sigma(x)=x$ for each free variable $x$. More generally, for each $w\in W$, $\sigma_w (x)=x$ for each $x\in Dom(\sigma)$.     

From Claim (3), if $(w: \Gamma, \sigma)$ appears in $\T$, then all the free variables are in $Dom(\sigma)\subseteq D_w=Dom(t_w)$, thus it makes sense to ask whether $\M, w, \sigma \vDash \Gamma$ (the assignment to the bound variables are irrelevant for the truth of formulas in $\Gamma$). To prove that $\M, w, \sigma \vDash \Gamma$ for all nodes $(w: \Gamma, \sigma)$ in $\T$, we do induction on the nodes of $\T$ in a bottom-up fashion from leaf nodes, by following the the rules conversely. 
\begin{itemize}
\item[Leaf] For leaf nodes $(w: l_1, \dots, l_k, \sigma)$, by definition of $\rho$, the statement holds based on the fact that $\sigma(x)=x$ for all the free variable $x$ in those literals by Claims (3), (6). 
\item[\END] Supposing $\M, w, \sigma \vDash l_1\land  \dots \land l_k$, then it is clear that $\M, w, \sigma \vDash \KSxf\phi_1 \land\dots\land\KSxl\phi_n\land l_1\land  \dots \land l_k$ since there is no outgoing transition from $w$, and $D_w$ is not empty.  
\item[$\lor$, $\land$] Obvious. 
\item[$\BR$] Suppose $\M, wv_{y_i}^y, \sigma'\vDash \psi_i[y\slash y_i]\land \bigwedge_{1}^{n}\phi_j$ for every $y\in Dom(\sigma')$ and $i\in [1, m]$, and $(w: \Gamma, \sigma)$ is the branching predecessor where $$\Gamma=\{\K^{x_j}\phi_j\mid j\in[1,n]\}\cup\{\hK^{y_i}\psi_i\mid i\in[1,m]\}\cup\{l_h\mid h\in[1,k]\}.$$ Note that $D_w=Dom(t_w)=Dom(\sigma')$. We need to show that $\M, w, \sigma \vDash \Gamma$. 
%\bigwedge_{1}^{n}\K^{x_j}\phi_j \land \bigwedge_{1}^{m}\hK^{x_i}\psi_i\bigwedge_{1}^{k} l_h.$ 
The $l_h$ part is as in the case of the leaf nodes. For $\K^{x_j}\phi_j$, let $x_j\in Dom(\sigma')=D_w$ be the witness, we need to show that at all the successors $wv_{y_i}^y$ of $w$, $\M, wv_{y_i}^y, \sigma[x_j\mapsto x_j]\vDash \phi_j$ . Now by Claim (4), all $x_k$ such that $k\not=j$ are not free in $\phi_j$. From the induction hypothesis (IH) that $\M, wv_{y_i}^y, \sigma'\vDash \phi_j$ for each $wv_{y_i}^y$, and the fact that $\sigma[x_j\mapsto x_j]$ and $\sigma'$ agree on the free variables in $\phi_j$, we know $\M, wv_{y_i}^y, \sigma[x_j\mapsto x_j]\vDash \phi_j$ for each $wv_{y_i}^y$. Thus $\M, w, \sigma\vDash \K^{x_j}\phi_j$.  

As for each $\hK^{y_i}\psi_i$, we show that for each $y\in D_w= Dom(\sigma')$, $\M, wv^y_{y_i}, \sigma[y_i\mapsto y]\vDash \psi_i$. Note that $y$ might be not in $Dom(\sigma)$. By the IH, $\M, wv_{y_i}^y, \sigma'\vDash \psi_i[y\slash y_i]$. Since $\sigma'$ assigns $y$ to $y$, we just need to show that $\sigma[y_i\mapsto y]$ and $\sigma'$ agree on all the free variables in $\psi_i[y\slash y_i]$ except $y$. Note that $y_i$ is not a free variable in $\psi_i[y\slash y_i]$, therefore the only possible  differences between $\sigma[y_i\mapsto y]$ and $\sigma'$ are about those $x_j$ where  $x_j\not=y$. By Claim (5), we know that $x_j$ is not a free variable in $\psi_i[y\slash y_i]$ if $x_j\not= y$. Therefore $\M, wv^y_{y_i}, \sigma[y_i\mapsto y]\vDash \psi_i$ for each $y\in D_w$, thus $\M, w,\sigma\vDash \hK^{y_i}\psi_i$ for each $i$. 
\end{itemize}
It follows that $\M,r,\sigma_r\vDash \theta$. 
\medskip

Now \textbf{from bottom to top}: We just need to show that the rule applications preserve the satisfiability of the formula set. Note that for $(\lor)$ it suffices to show one outcome node is still satisfiable. In this way, there is an open tableau since if the formula sets at the leaf nodes and branching nodes are satisfiable,  then there is no contradiction among the literals. It is obvious that $(\land)$ and $(\END)$ preserves satisfiability, and one of the denominator of $\lor$ preserves it too. We now show that $\BR$ also does so. Supposing $\Gamma=\{\KSxf\phi_1, \dots, \KSxl\phi_n,\hKSyf\psi_1,\dots, \hKSyl\psi_m, \vec{l}\}$, in some branching node $(w: \Gamma, \sigma)$, is satisfiable, then there is a model $\M, w$ and an assignment $\eta$ such that $\eta(x)\in D_w$ for all $x\in\Var$ and: $$\M, w, \eta \vDash \{\KSxf\phi_1, \dots, \KSxl\phi_n,\hKSyf\psi_1,\dots, \hKSyl\psi_m\} \qquad (\circ)$$
By the semantics, we know there are $a_1, \dots, a_n\in D_w$, such that for all the successors $v$ of $w$:
$\M, v, \eta[x_j\mapsto a_j] \vDash \phi_j$. Due to Claim (4), each $x_j$ is not free in $\phi_k$ for $k\not=j$ thus we can safely obtain for all successor $v$ of $w$: $$\M, v, \eta[\vec{x}\mapsto \vec{a}] \vDash \{\phi_1, \dots, \phi_n\}. \qquad (A)$$

Now also by $(\circ)$ and the semantics for $\hK^{y_i}$, for each $\psi_i$ and each $b\in D_w$, there is a successor $v^b_i$ of $w$ such that 
$$\M, v^b_i,\eta[y_i\mapsto b]\vDash  \psi_i \qquad (B)$$ 

From Claim (5), $y_i$ is not one of $x_j$. Now from 
(A) and (B), we have for each $\psi_i$ and each $b\in D_w$, there is a successor $v^b_i$ of $w$ such that:   
$$\M, v^b_i,\eta[\vec{x}\mapsto\vec{a}, y_i\mapsto b]\vDash \{\phi_1, \dots, \phi_n, \psi_i\} \quad  (\star)$$ 
Now consider any $y\in Dom(\sigma')$ as in the denominator of the rule $\BR$, we need to show $\{\phi_1, \dots, \phi_n, \psi_i[y\slash y_i]\}$ is satisfiable. Suppose that $y\in Dom(\sigma')$ and  $\eta(y)=b\in D_w$. By Claim (2), $y$ is free in $\psi_i[y\slash y_i]$, thus by ($\star$):
$$\M, v^b_i, \eta[\vec{x}\mapsto\vec{a}]\vDash \{\phi_1, \dots, \phi_n, \psi_i[y\slash y_i]\}$$
% Suppose that $y\in Dom(\sigma')$ and $y$ is one of $\vec{x}$, say $x_j$, which is assigned $a_j$. By Claim (2), $y$ is free in $\psi_i[y\slash y_i]$, thus by ($\star$) and let  $b=a_j$: 
% $$\M, v^{a_j}_i, \eta[\vec{x}\mapsto\vec{a},y\mapsto a_j]\vDash \{\phi_1, \dots, \phi_n, \psi_i[y\slash y_i]\}$$
% Similarly, if $y$ is not one of $\vec{x}$, then supposing  $\eta(y)=b\in D_w$, by $(\star)$ we have: 
% $$\M, v^b_i, \eta[\vec{x}\mapsto\vec{a}]\vDash \{\phi_1, \dots, \phi_n, \psi_i[y\slash y_i]\}$$

Therefore, $\{\phi_1, \dots, \phi_n, \psi_i[y\slash y_i]\}$ is satisfiable for each $y\in Dom(\sigma').$ This completes the whole proof. 
\end{proof}
% Note that the above result can also be generalized to the full $\MLMS$ with multiple modalities labelled by different $i\in \Ag$. We need to generalize the rule $\BR$ by making different sections for each $i\in\Ag$ in the numerator. The construction for the model is adapted accordingly.    
A strong finite tree property then follows: \footnote{Interested readers may go back to Footnote \ref{ft.inv}.}
\begin{corollary}\label{co.sat}
If an $\MLMS$-formula $\phi$ (with length $n$) is satisfiable, then it is satisfiable in a finite tree model $\M,w$, w.r.t.\ an assignment $\sigma$ such that
$|W^\M|\leq n^{4n}$, $|D^\M|\leq n$, $\rho(P, w)\subseteq D_w^k$ (if $P$ is $k$-ary), the depth of the tree is bounded by $2|\phi|$, and $\sigma(x)\in D^\M_w$ for each free variable $x$ in $\phi$.  
\end{corollary}
The upper bound for the depth of the tree comes from the bound on the length of the PNF of $\phi$. The (very loose) upper bound on the size of $W^\M$ comes from the fact that each node in the tableau may contain up to $n$ modalities and each modality may have $n$-successors (due to the size of the domain), and the depth of the tableau is up to 2$n$. As in normal modal logic \cite[Ch.\ 6]{mlbook}, we can force a binary branching tree by an $\MLMS$-formula to get an exponential-sized model. 

% Note that the satisfiability problem of \MLMS-formulas is \PSPACE-hard since we can code the propositional modal logic in $\MLMS$ with 0-ary predicates, and 
%Based on the tree models, a fairly standard $\PSPACE$-algorithm would suffice to show the following: 
\begin{theorem}
Satisfiability problem of \MLMS-formulas is \PSPACE-complete. 
\end{theorem}
\begin{proof}(Sketch)
Note that standard modal logic formulas can be translated into our $\MLMS$ by using 0-ary predicate $P$ for each propositional letter. For example, $\Box \Diamond p$ can be translated as $\KS \Diamond^y P$. The \PSPACE-hardness then follows since satisfiability for basic normal modal logic is \PSPACE-complete \cite{mlbook}. For the upper bound, note that to rewrite a formula into PNF and to reletter the formula into a clean shape are efficient in space, and the length of the resulting formula is still linear in the length of the original formula. From then on, a standard \PSPACE-algorihm traversing a tree structure including all the the possible $\lor$-branches suffices, as in standard modal logic, with some care about efficiently encoding the descriptions of each node and the result of the consistent checking.   
\end{proof}
\begin{remark}
Our notion of the tableaux is closely related to two-player satisfiability games and alternating-tree automata \cite{LNCS2500}, which will give us the algorithmic tools for $\MLMS$ in the future.
\end{remark}
We conjecture that a similar tableau method would work for $\MLMSEQ$, with more careful assignment management by selecting representatives of provably equivalent variables w.r.t. $\approx$ as the local domain. We leave it to the extended version of the paper.

\section{A new epistemic logic } \label{sec.el}
In this session, we go back to the motivating epistemic setting, and give a complete axiomatization of our logic over epistemic models.
\subsection{Epistemic language}
To ease the presentation of the proof system,  we also include the standard modality $\Box$ as a primitive modality in the following language $\MLMSKEQ$:\footnote{In the axiomatization we will make use of $\Box$. If we do not introduce them explicitly then every time we need to use $\KS\phi$ where $x$ is not free in $\phi$.}
$$ \phi::= x\approx y \mid P\vec{x} \mid \neg\phi \mid (\phi\land\phi)\mid \K\phi\mid \KS\phi $$
Recall that $\MLMSKEQ$ is equally expressive as $\MLMSEQ$. In the epistemic setting, the intended reading of $\Box\phi$ is \textit{$i$ knows that $\phi$}, and the intended reading of $\KS\phi$ is that \textit{$i$ knows something such that $\phi$.} The semantics is as before, but we have two extra conditions on the model: 
\begin{itemize}
\item Each $R$ is an \textit{equivalence relation} as in (idealized) standard epistemic logic S5. We also write $\sim$ for $R$. 
\item For all $w\in W$, $D_w=D$.
\end{itemize}
Note that for an increasing domain model $\M$, if $R$ is an equivalence relation, then for any two worlds $w,v$ such that $wRv$, we have $D_w=D_v$. Therefore, we can simply assume that there is a \textit{constant domain} over the whole model.\footnote{Without the condition of constant domain, different partitions w.r.t.\ $\sim$ may still have different local domains.} We also call such models \textit{S5-models}, following the terminology in epistemic logic. $\phi\in\MLMSKEQ$ is \textit{valid} if for any pointed S5-model $\M,w$, any assignment $\sigma$, $\M,w,\sigma\vDash\phi.$ Note that in constant domain models, we no longer need to give the restriction on the assignment: $\sigma(x)\in D=D_w$ for all assignment $\sigma$, all $w\in W_\M$ and all $x\in \Var$. 

$\MLMSKEQ$ over S5-models is a quite powerful language. As we mentioned, the semantics for $\KS$ is in line with the \textit{mention-some} reading. We can also introduce a modal operator $\KA$ based on \textit{mention-all} semantics as below: 
$$\begin{array}{|rcl|}
\hline 
 \M, w, \sigma\vDash \KA \phi 
%&\Leftrightarrow& \text{for each $d\in D$,} \text{ \textit{either}: } \\ 
% &&  \text{for all $t\sim s$, } \M, t, \sigma[x\mapsto d]\vDash\phi\\ &&  \text{\textit{or} for all $t\sim s$ }  \M, t, \sigma[x\mapsto d]\vDash\neg \phi \\
&\Leftrightarrow& \text{for each $d\in D$, \textit{either}}\M, w, \sigma[x\mapsto d]\vDash\Box\phi \text{ \textit{or} } \M, w, \sigma[x\mapsto d]\vDash\Box\neg \phi \\
\hline
\end{array}$$
Intuitively, $\KA\phi$ means for all objects in the constant domain $D$, the agent \textit{knows whether} $\phi$, e.g., \textit{I know who came to the party} means that for each person in concern, I know whether (s)he came or not. In terms of $\FOML$, $\KA\phi$ is essentially $\forall x (\Box\phi\lor \Box\neg \phi).$ Note that $\KS$ and $\KA$ not only differ only in the quantifiers. On the other hand, the natural $\forall$-version of $\KS$ is defined:
$$\begin{array}{|rcl|}
\hline 
\M, w, \sigma\vDash \Box^{\forall x} \phi &\Leftrightarrow& \text{for each $d\in D$, } \M, w, \sigma[x\mapsto d]\vDash\K\phi\\
\hline
\end{array}$$
$\M, w, \sigma\vDash \Box^{\forall x} \phi$ only says that of each object, the agent knows that it satisfies property $\phi(x)$, which differs from the semantics of the \textit{mention-all} operator.

There is also a very natural generalization of our  mention-some operator with multiple variables: $\K^{\vec{x}}$: 
$$\begin{array}{|rcl|}
\hline 
\M, w, \sigma\vDash \K^{x_1\cdots x_n} \phi &\Leftrightarrow& \text{there exist $d_1,\cdots, d_n\in D$ } \text{ such that } \M, t, \sigma[\vec{x}\mapsto \vec{d}]\vDash\Box\phi\\
\hline
\end{array}
$$

\begin{proposition}
$\KA,\Box^{\forall x}$, and $\K^{\vec{x}}$ can all be defined in $\MLMSKEQ$ over S5-models: 
% \begin{itemize}
% %\item Standard know-that operator $\K$
% \item Mention-all operator $\KA$
% \item 
% \item Generalized mention-some operator $\K^{\exists\vec{x}}_i$
% \end{itemize}
\end{proposition}
\begin{proof}
The following equivalences are valid for any $\MLMSKEQ$-formula $\phi$, based on the fact that $\Diamond\Box\phi\lra\Box\phi$ and $\Box\Box\phi\lra\Box\phi$ are valid on S5 models: 
\begin{itemize}
\item $\KA\phi\lra (\hK^x (\K\phi\lor \K\neg \phi))$
\item $\Box^{\forall x}\phi\lra \hKS\Box\phi$ 
\item $\K^{\vec{x}}\phi\lra \K^{x_1}\cdots \K^{x_n}\phi$
\end{itemize}
% Note that for all $\phi \in \MLMS$: 
% $$\vDash \KA\phi\lra \neg \KS\neg \K(\K\phi\lor \K\neg \phi) $$
% It can be further simplified by using S5 properties. 
% $$\vDash \KA\phi\lra \neg \KS\neg (\K\phi\lor \K\neg \phi) $$
% Namely, 
% $$\vDash \KA\phi\lra \hK^x_i (\K\phi\lor \K\neg \phi) $$
% As for $\K^{\vec{x}}$, it is not hard to see that for all $\phi\in \MLMS$
% $$\vDash \K^{\exists\vec{x}}\phi\lra \K^{x_1}\cdots \K^{x_n}\phi$$
\end{proof}

\subsection{Proof system}
In the rest of this section we propose a Hilbert-style proof system for the equality-free $\MLMSK$ and then extend it to a proof system for $\MLMSKEQ$. 

First note that the following ``K axiom'' for $\KS$ is \textit{not} valid on S5-models, since the witnesses object for $x$ in   $\KS(\phi\to\psi)$ and $\KS\phi$ may be different: 
$$ \KS(\phi\to\psi)\to (\KS\phi\to \KS\psi)$$
Therefore the modal logic of $\MLMSK$ is \textit{not normal}, as  expected. We propose the following proof system: 
%  Note that, essentially, $\K\phi$ in the following axiom schemata represents any $\KS\phi$ such that $x$ does not appear freely in $\phi$. 
\begin{center}
	\begin{tabular}{ll}
		\multicolumn{2}{c}{System $\SMLMSK$}\\
		% \lcline{1-2}\rcline{3-4}
		{\textbf{Axioms}}&\\
		%\lcline{1-2}\rcline{3-4}
		\TAUT & \text{all axioms of propositional logic}\\
		\DISTK & $\K(\phi\to\psi)\to (\K\phi\to \K\psi)$\\
		\AxTrK& $\K\phi\to \phi $ \\
		\AxTrans& $\KS \phi\to\K\KS \phi$\\
		\AxEuc& $\neg \KS \phi\to\K\neg\KS \phi$\\
		\AxKtoMS &$\K (\phi[y\slash x]) \to \KS \phi$ \text{(if $\phi[y\slash x]$ is admissible)}\\	
         \AxMStoK  & $\KS\phi\to \K\phi$ \text{(if $x\not\in FV(\phi)$)}\\
\AxMStoMSK&$\KS \phi \to \KS\K \phi$  \\
        \AxKT& $\K\top $
%         \AxId & $x=x$\\
%         \AxSub & $x=y \to (\phi(x)\lra\phi[x\slash y])$ if $y$ does not occur in $\phi$\\
%         \AxKE &$x=y\to \K(x=y)$
\end{tabular}
	\begin{tabular}{lclc}
\textbf{Rules:}\\
 \MP & $\dfrac{\varphi,\varphi\to\psi}{\psi}$&\MONO& $\dfrac{\vdash\varphi\to\psi}{\vdash\KS\varphi\to\KS\psi}$ 
%& \SUB & $\dfrac{\varphi(p)}{\varphi[\psi\slash p]}$
\end{tabular}
\end{center}
where $\phi[y\slash x]$ is \textit{admissible} if no free occurrence of $x$ in $\phi$ is in the scope of a modality binding $y$. 

The axioms are quite intuitive as they are, but they can be understood even more easily if we put them in some context to give concrete intuitive readings. For example, supposing we use $\KS\phi(x)$ to express that the agent \textit{knows how to achieve  $\phi$}, then $\AxTrans$, $\AxEuc$ express the (idealised) introspections of such goal-directed know-how: if you know how then you know you know how, and if you do not know how then you know you do not; $\AxKtoMS$ says that if you know that a particular way can achieve $\phi$ for sure, then you know how to achieve $\phi$; $\AxMStoK$ trivializes know-how to know-that, if $x$ is irrelevant to $\phi$; $\AxMStoMSK$ says that if you know how to achieve $\phi$ then you know how to ensure that you know $\phi$; $\AxKT$ says we know tautologies;\footnote{This is a technical axiom to recover the necessitation rule for $K$ from $\MONO$. Alternatively, we can also just include the necessitation rule for $\K$ instead of $\AxKT$.} and $\MONO$ is the monotonicity rule for know-how. 

\begin{remark}
Actually, all the above axioms have incarnations in the logic system $\mathbb{SKH}$ of know-how proposed in \cite{KH17}. Modulo the replacement of the know-how operator $\Kh$ by $\KS$, all the axioms and rules in $\mathbb{SKH}$ are derivable in $\SMLMSK$, except the composition axiom $\mathtt{KhKhtoKh}$ ($\Kh\Kh\phi\to\Kh\phi$), which indeed captures the characteristic feature about know-how: the witness strategies can be composed. %$\neg \KS\bot$ can be derived from $\neg \K\bot$.    
\end{remark}

It is routine to verify the soundness, assuming the familiarity of first-order modal logic over S5-models.  
\begin{theorem}
$\SMLMSK$ is sound over S5 models. 
\end{theorem}
$\SMLMSK$ is very powerful, as demonstrated by the following proposition: many usual ``suspects'' can be derived. 
\begin{proposition}
The following theorems are derivable and the rules are admissible in $\SMLMSK$: \\
%\scalebox{0.95}{
\begin{center}
\begin{tabular}{llll}
$\AxMSKtoMS$ & $\KS\K\phi\to\KS\phi$ & $\AxMStotK$ & $\K\phi\lra \KS\phi\ (x\not\in FV(\phi))$ \\
$\NECK$ & $\dfrac{\vdash\varphi}{\vdash\K\varphi}$ & \NECMS &$\dfrac{\vdash\varphi}{\vdash\KS\varphi}$\\
\RE & $\dfrac{\vdash \varphi\lra\psi}{\vdash\chi(\psi)\lra \chi(\phi)}$ &$\KtoMS$  & {$\dfrac{\vdash \K\varphi\to\psi}{\vdash\KS\varphi\to\psi}$ $(x\not\in FV(\psi))$} \\
\AxMST & $\KS\top$  & $\AxRMS$ & $\KS\phi\lra \KSy\phi[y\slash x]\ (y \text{ is not in } \phi))$\\ 
\AxTransK & $\K \phi\to\K\K \phi$ & \AxEucK & $\neg \K \phi\to\K\neg\K\phi$
\end{tabular}
\end{center}
%}
\end{proposition}
\begin{proof}(Sketch)
$\AxMSKtoMS$ is based on $\AxTrK$ and $\MONO$. $\K\phi\to \KS\phi$ is a special case of $\AxKtoMS$ when $y=x$. Together with $\AxMStoK$ we have $\AxMStotK$. Now from $\AxKT$ and $\AxMStotK$ we have $\AxMST$. Then based on $\AxKT$ and $\AxMST$,  $\NECMS$ and $\NECK$ follow from $\MONO$. The rule of Replacement of Equals ($\RE$) can be proved inductively based on $\MONO$. $\KtoMS$ is based on $\MONO$, $\AxMStoMSK$, $\AxMStotK$ and $\AxTrK$.  The standard introspection axioms of $\K$ are special instances of inspection axioms $\AxTrans$ and $\AxEuc$ based on $\AxMStotK$ and $\RE$.\footnote{$\AxTrans$ can also be derived from $\AxTrK$ and $\AxEuc$ in the system $\SMLMSK$ like in the case of propositional S5 system.} Let us now look at $\AxRMS$, which is the renaming axiom for bound variables. Right-to-left can be derived by starting with $\AxKtoMS$: $\vdash \Box(\phi[y\slash x])\to \KS \phi$, then since $y$ does not appear in $\phi$, we have $\vdash \KSy(\phi[y\slash x])\to \KS\phi$ by $\KtoMS$. For the other direction, note that $\phi=(\phi[y\slash x])[x\slash y]$ if $y$ is not in $\phi$. Then since $(\phi[y\slash x])[x\slash y]$ is admissible, by $\AxKtoMS$, $\vdash\K(\phi[y\slash x])[x\slash y]\to\KSy \phi[y\slash x]$, namely $\vdash\K\phi \to\KSy \phi[y\slash x].$ Then by $\KtoMS$, $\vdash \KS \phi \to  \KSy\phi[y\slash x]$.   
\end{proof}
% \begin{proof}(Sketch)
% $\AxMSKtoMS$ is based on $\AxTrK$ and $\MONO$. $\K\phi\to \KS\phi$ is a special case of $\AxKtoMS$ when $y=x$. Together with $\AxMStoK$ we have $\AxMStotK$. Now from $\AxKT$ and $\AxMStotK$ we have $\AxMST$. Then based on $\AxKT$ and $\AxMST$,  $\NECMS$ and $\NECK$ follow from $\MONO$. The rule of Replacement of Equals ($\RE$) can be proved inductively based on $\MONO$. $\KtoMS$ is based on $\AxMStotK$ and $\RE$.  The standard introspection axioms of $\K$ are special instances of inspection axioms $\AxTrans$ and $\AxEuc$ based on $\AxMStotK$ and $\RE$. Let us look at $\AxRMS$, which is the renaming axiom for bound variables. Right-to-left can be derived by starting with $\AxKtoMS$: $\vdash \Box(\phi[y\slash x])\to \KS \phi$, then applying $\MONO$ to have $\vdash \KSy\Box(\phi[y\slash x])\to \KSy \KS\phi$. Since $y$ does not appear in $\phi$, we have $\vdash \KSy\K(\phi[y\slash x])\to \K\KS\phi$ by $\AxMStotK$ and $\RE$. Finally we have $\vdash \KSy(\phi[y\slash x])\to \KS\phi$ by $\AxMStoMSK$ and $\AxTrK$. For the other direction, note that $\phi=(\phi[y\slash x])[x\slash y]$ if $y$ is not in $\phi$. Then since $(\phi[y\slash x])[x\slash y]$ is admissible, by $\AxKtoMS$, $\vdash\K(\phi[y\slash x])[x\slash y]\to\KSy \phi[y\slash x]$, namely $\vdash\K\phi \to\KSy \phi[y\slash x].$ Then by $\KtoMS$, $\vdash \KS \phi \to  \KSy\phi[y\slash x]$.   
% \end{proof}
\begin{remark}Now we know that the S5 system of $\Box$ is a subsystem of $\SMLMSK$. We may also view $\AxKtoMS$, $\KtoMS$ as analogues of the following axiom and rule respectively in first-order logic: 
$$\forall x \phi \to \phi[y\slash x] \qquad \dfrac{\vdash \phi\to\psi}{\vdash \phi\to\forall x\psi} \text{ $(x\not\in FV(\phi))$}$$
$\RE$ is the rule of replacement and $\AxRMS$ allows us to reletter the bound variables, as in \FOL. Moreover, the Barcan formula ($\forall x\Box \phi\to \Box\forall x\phi$), over S5-models, can be expressed as: $$ \hKS\K\phi \to \K\hKS\K\phi$$ which is also derivable in $\SMLMSK$ due to $\AxEuc$.    
\end{remark}
The following provable formula plays a role in the later completeness proof.
\begin{proposition}\label{prop.kimp}
$\vdash_{\SMLMSK} (\KS\phi\to \K\psi) \to \K(\KS\phi\to \K\psi)$.
\end{proposition}
\begin{proof}
Consider the contrapositive of the formula to be proved: 
%$$\neg \K(\KS\phi\to \K\psi)\to \neg (\KS\phi\to \K\psi).$$ Namely,
$\hK(\KS\phi \land \neg \K\psi)\to (\KS\phi\land  \neg \K\psi)$.
It is routine to show that $\vdash \hK(\KS\phi \land \neg \K\psi) \to (\hK\KS\phi \land \hK\neg \K\psi)$. Now by  $\AxEuc$, $\AxTransK$, and $\AxTrK$, $\vdash \hK\KS\phi\lra \KS\phi$ and $\vdash \hK\neg \K\psi\lra \neg \K\psi$. Therefore, by \RE, $\vdash \hK(\KS\phi \land \neg \K\psi)\to (\KS\phi\land  \neg \K\psi).$ 
\end{proof}
%Note that if $x, y\not\in FV(\phi)$ then $\AxRMS$ amounts to $\vdash \KS\phi\lra \KSy\phi$. %It also justifies our notation of $\K\phi$: the exact $x$ does not matter.
In the following we define the Hilbert system $\SMLMSKEQ$ for $\MLMSKEQ$ by extending $\SMLMSK$ with the following two extra axiom schemata:
\begin{center}
	\begin{tabular}{ll}
%		\multicolumn{2}{c}{System $\SMLMSEQ$}\\
		% \lcline{1-2}\rcline{3-4}
		{\textbf{Axioms}}&\\
		%\lcline{1-2}\rcline{3-4}
% 		\TAUT & \text{all axioms of propositional logic}\\
% 		\DISTK & $\K(\phi\to\psi)\to \K\phi\to \K\psi$\\
% 		\AxTrK& $\K\phi\to \phi $ \\
% 		\AxTrans& $\KS \phi\to\K\KS \phi$\\
% 		\AxEuc& $\neg \KS \phi\to\K\neg\KS \phi$\\
% 		\AxKtoMS &$\K (\phi[y\slash x]) \to \KS \phi$ \\	
% 		\AxMStoMSK&$\KS \phi \to \KS\K \phi$  \\
%         \AxKT& $\KS\top $\\
        \AxId & $x\approx x$\\
        \AxSub & $x\approx y \to (\phi\to\psi)$ if $\phi$ and $\psi$ only differ in that\\ & some free occurrences of $x$ in one formula \\
        &are replaced by free occurrences of $y$ in another.\\
%        \AxKE &$x=y\to \K(x=y)$
\end{tabular}
% 	\begin{tabular}{lclc}
% \textbf{Rules:}\\
%  \MP & $\dfrac{\varphi,\varphi\to\psi}{\psi}$&\MONO& $\dfrac{\vdash\varphi\to\psi}{\vdash\KS\varphi\to\KS\psi}$ 
% %& \SUB & $\dfrac{\varphi(p)}{\varphi[\psi\slash p]}$
% \end{tabular}
\end{center}
\begin{theorem}
$\SMLMSKEQ$ is sound. 
\end{theorem}
It is routine to show: 
\begin{proposition}
The following are provable in $\SMLMSKEQ:$\\
\begin{tabular}{llll}
\AxSym & $x\approx y\to y \approx x$ & \AxTranseq & $x \approx y\land y\approx z \to x\approx z$\\
\AxKEQ & $x\approx y\to \K (x\approx y)$ &\AxKNEQ & $ x\not \approx y\to \K (x\not \approx y)$ \\
\end{tabular}
\end{proposition}
\begin{proof} $\AxSym$ and $\AxTranseq$ are due to $\AxSub$. To derive $\AxKEQ$, we first have $\vdash\K (x\approx x)$ by $\AxId$ and $\NECK$. By $\AxSub$ we have $\vdash x\approx y\to (\K (x\approx y)\to \K (x\approx x)).$ Therefore $\AxKEQ$ is provable. 
 
For $\AxKNEQ$, from $\AxKEQ$ we have $\vdash \hK x\not\approx y\to x\not\approx y$, thus $\vdash\Box\hK x\not\approx y\to \K  x\not\approx y$ by $\NECK$ and $\DISTK$. Note that $\phi\to \K\hK\phi$ is derivable by $\AxTrK$ and $\AxEuc$. Therefore by taking $\phi=x\not\approx y$ and using $\MP$ we have $\vdash\AxKNEQ$. \end{proof}

% \noteYW{In the following if we say $\Gamma$ is consistent, we mean $\Gamma$ is $\SMLMS$-consistent. }

\section{Completeness} \label{sec.comp}
Let us first focus on the completeness of $\SMLMSK$ without axioms for equalities.  

We first extend the language of $\MLMSK$ with countably infinitely  many new variable symbols. Call the new language $\MLMSK^+$ and the variable set $\Var^+$. In the following, we say a set of formulas is \textit{$\SMLMSK^+$-consistent} if it is $\SMLMSK$-consistent w.r.t. the extended language $\MLMSK^+$. 
%We call the corresponding proof system $\SMLMSK^+$, which can use variables in $\Var^+$.  

\begin{definition}
A set of $\MLMSK^+$ formulas has $\exists$-property if for each $\KS\phi\in \MLMSK^+$ it contains a ``witness'' formula $\KS\phi \to \K\phi[y\slash x]$ for some $y\in \Var^+$ where $\phi[y\slash x]$ is admissible. 
%where $y$ is a fresh variable which does not appear in $\phi$. 
\end{definition}

\begin{definition}[Canonical model for $\SMLMSK^+$]\label{def.cm}
The canonical model is a tuple $\lr{W^c, D^c, \sim^c, \rho^c}$\\
 where: 
\begin{itemize}
\item $W^c$ is the set of maximal $\SMLMSK^+$-consistent sets  with $\exists$-property,
\item $D^c=\Var^+$,
\item $s\sim^c t$ iff $\K(s)\subseteq t$ where $\K(s):=\{\phi\mid \K\phi\in s\}$,
\item $ \vec{x}\in \rho^c(P,s)$ iff $P\vec{x}\in s$.
\end{itemize}
\end{definition}
It is routine to show that $\sim^c$ is an equivalence relation, by using axioms $\AxTrK$, $\AxTransK$, and $\AxEucK$.

%\noteYW{anything wrong with multi-agent? There are indeed complications!!! But should be solvable!! }
\begin{lemma}\label{lem.ex}
If $\K\psi\not\in s \in W^c$ then there exists a $t\in W^c$ such that $s\sim^ct$ and $\neg \psi\in t$.  
\end{lemma}
\begin{proof}
It is routine in normal modal logic to show that if $\K\psi\not\in s$ then $\K(s)\cup\{\neg \psi\}$ is consistent. Now we show that $\K(s)\cup\{\neg \psi\}$ can be extended to an $\SMLMSK^+$-maximal consistent set with $\exists$-property. We follow the general strategy in \cite{Cresswell96} by adding witness formulas one by one. 

Let $\theta_0=\neg \psi$. We enumerate $\KS$-formulas as: $\KSxf\phi_1, \K^{x_2}\phi_2, \dots$ We define $\theta_{k+1}$ as the formula: 
$$\theta_k\land (\K^{x_{k+1}}\phi_{k+1}\to \K\phi_{k+1}[y\slash x_{k+1}])$$
\noindent where $y$ is the first variable in a fixed enumeration of $\Var^+$ such that $\phi_{k+1}[y\slash x_{k+1}]$ is admissible and $\K^{x_{k+1}}\phi_{k+1}\to \K\phi_{k+1}[y\slash x_{k+1}]$ is consistent with $\{\theta_k\}\cup \K(s)$. 

We now show that such $\theta_{k+1}$ always exists (i.e., such a $y$ exists), if $\{\theta_k\} \cup \K(s)$ is consistent. Towards a contradiction, suppose that $\{\theta_k\}\cup \K(s)$ is consistent but there is no such a $y$, i.e., for each $y\in\Var^+$ such that $\phi_{k+1}[y\slash x_{k+1}]$ is admissible there are $\chi_1,\dots,\chi_n\in \K(s)$ such that $$\vdash \chi_1\land \dots \land \chi_n\to ((\K^{x_{k+1}}\phi_{k+1}\to \K\phi_{k+1}[y\slash x_{k+1}])\to \neg\theta_k).$$

By $\NECK$ and $\DISTK$ and the fact that $\K\chi_i\in s$ for $1\leq i\leq n$, it is routine to show that $\K((\K^{x_{k+1}}\phi_{k+1}\to \K\phi_{+1}[y\slash x_{k+1}])\to \neg \theta_k)$ is also in $s$. By $\DISTK$ again, 
$$\K(\K^{x_{k+1}}\phi_{k+1}\to \K\phi_{k+1}[y\slash x_{k+1}])\to \K \neg \theta_k \quad (\star)$$ is in $s$ for all $y$ such that $\phi_{k+1}[y\slash x]$ is admissible. 

Note that $s$ has the $\exists$-property, therefore, $\K^{x_{k+1}}\phi_{k+1}\to \K\phi_{k+1}[y^*\slash x]$ is in $s$ for some particular $y^*$ such that $\phi_{k+1}[y^*\slash x_{k+1}]$ is admissible. By Proposition \ref{prop.kimp}, $\K(\K^{x_{k+1}}\phi_{k+1}\to \K\phi_{k+1}[y^*\slash x_{k+1}])$ is in $s$. By ($\star$), we have $\K \neg \theta_k$ is in $s$, thus $\neg\theta_k\in \K(s)$ which is in contradiction with that $\{\theta_k\}\cup\Box(s)$ is consistent. 

Now since $\{\theta_0\} \cup \K(s)=\{\neg \psi \}\cup \K(s)$ is consistent, we can indeed construct all the $\theta_k$. Let $\Gamma=\{\theta_k \mid k\in\mathbb{N}\}\cup \Box(s).$ $\Gamma$ is consistent and the $\exists$-property is essentially built-in.  We can then extend it into an $\SMLMSK^+$-maximal consistent set $t$. It then follows that $s\sim^ct$ and $\neg \phi\in t$.   
\end{proof}
Now comes the truth lemma. 
\begin{lemma}\label{lem.truth} Let $\sigma^*$ be the assignment such that $\sigma^*(x)=x$ for all $x\in\Var^+$. For any $\phi\in\MLMSK^+$, any $s\in W^c$: 
$$\M^c, s, \sigma^* \vDash \phi \iff \phi \in s $$
\end{lemma}
\begin{proof}
We do induction on the structure of the formula $\phi$. 
\begin{itemize}
\item[$P\vec{x}$] By the definition of $\rho^c$ and $\sigma^*$, it is obvious. 
\item[Bool]Boolean cases are trivial.  
\item[$\K\psi$] Routine, based on Lemma \ref{lem.ex}.  
\item[$\KS\psi$]Supposing $\KS\psi\in s$, by $\exists$-property of $s$ there exists $\K\psi[y\slash x]\in s$ for some $y$ such that $\psi[y\slash x]$ is admissible. By $\AxTransK$, $\K\K\psi[y\slash x]\in s$. Supposing $s\sim^c t$, we have $\K\psi[y\slash x]\in t$ due to the definition of $\sim^c$. By $\AxTrK$, $\psi[y\slash x]\in t$. By IH, $\M, t, \sigma^* \vDash \psi[y\slash x]$. Due to the definition of $\sigma^*$ and the fact that $\psi[y\slash x]$ is admissible, $\M, t, \sigma^*[x\mapsto y] \vDash \psi$ for all $t\sim s$. Thus $\M,s, \sigma^*[x\mapsto y]\vDash \Box\psi$, it follows that $\M, s, \sigma^* \vDash \KS \psi$. 

Suppose $\KS\psi\not\in s$, by $\AxKtoMS$, $\K\psi[y\slash x]\not\in s$ for each $y$ such that $\psi[y\slash x]$ is admissible. By IH, 
%Lemma \ref{lem.ex} and the induction hypothesis, for each such a $y$ there exists a world $t$ such that $s\sim^c t$ and $\M^c, t, \sigma^*\vDash \neg \psi[y\slash x]$. Then 
$\M^c, s, \sigma^*\vDash \neg\K\psi[y\slash x]$ for each $y$ such that $\psi[y\slash x]$ is admissible. Due to the special assignment $\sigma^*$ such that $\sigma^*(y)=y$, it is  clear that $\M^c, s, \sigma^*[x\mapsto y]\vDash \neg\K\psi$ for each $y$ such that $\psi[y\slash x]$ is admissible. 

Now consider any $y'$ such that $\psi[y'\slash x]$ is not admissible, then by $\AxRMS$, we can reletter the modalities of $y'$ in $\psi$ with some fresh variable to obtain $\psi'$ such that $\vdash\psi\lra\psi'$ and $\psi'[y'\slash x]$ is admissible. Now by $\RE$, $\KS\psi'\not\in s$. We can then repeat the reasoning above to obtain $\M^c, s, \sigma^*[x\mapsto y']\vDash \neg\K\psi'$ for this $y'$. Since $\vdash\psi\lra\psi'$, by soundness, $\M^c, s, \sigma^*[x\mapsto y']\vDash \neg\K\psi$. Therefore for each $y$, no matter whether $\psi[y\slash x]$ is admissible,  we have $\M^c, s, \sigma^*[x\mapsto y]\vDash \neg\K\psi$.
Therefore $\M^c, s, \sigma^* \vDash \neg \KS\psi$.  
\end{itemize}  
\end{proof}

Note that not every $\SMLMSK^+$-consistent set of formulas can be extended into a world in $\M^c$, e.g., $\{\KS\phi(x)\}\cup\{\neg \K\phi(x)\mid x\in \Var^+\}$ cannot be extended consistently to obtain the $\exists$-property. However, every $\SMLMSK$-consistent set can be extended into a world in $W^c$ by adding the witness one by one using the new variables (cf. \cite{Cresswell96}). 
\begin{lemma}\label{lem.mcs}
Every $\SMLMSK$-consistent set of $\MLMSK$-formulas can be extended into an $\SMLMSK^+$-maximal consistent set of $\MLMSK^+$-formulas with $\exists$-property. 
\end{lemma}
Now based on Lemma~\ref{lem.mcs} and Lemma~\ref{lem.truth}, every $\SMLMSK$-consistent set is satisfiable by some pointed model and an assignment. The completeness follows: 
\begin{theorem}
$\SMLMSK$ is strongly complete w.r.t.\ $\MLMSK$ over S5 models.  
\end{theorem}
The completeness of $\SMLMSKEQ$ is quite routine based on the completeness proof of $\SMLMSK$. We only sketch the idea here following \cite{Cresswell96}. 
\begin{theorem}
$\SMLMSKEQ$ is strongly complete w.r.t.\ $\MLMSKEQ$ over S5 models.  
\end{theorem}
%\begin{proof}
%(Sketch) We first build a model as in Definition \ref{def.cm} but now w.r.t.\ $\SMLMSKEQ_+$, i.e., the system for a richer language with countably many new variable symbols. Then for each maximal consistent set $\Gamma$ in it, we take its generated subframe as the frame of our canonical model. Note that due to the definition of $\sim^c$ and $\AxKEQ$ and $\AxKNEQ$, $x\approx y\in \Gamma$ implies $x\approx y$ in all the worlds in the generated subframe, and $\neg x\approx y\in \Gamma$ implies $\neg x\approx y$ in all the worlds in the generated subframe. In another words, the worlds agree on the equalities. Now $\approx$ can be viewed as an equivalence relation between variables due to $\AxId, \AxSym$, and $\AxTranseq.$ We can define the domain $D$ as the set of equivalence classes according to $\approx$. Then we can define a canonical assignment such that $x$ is mapped to $[x]_{\approx}$. Correspondingly, we can define $\rho^c$ such that $\rho^c(P,s)(\vec{[x]})\iff P\vec{x}\in s.$ Then by similar strategy as in the case of $\SMLMS$, we can show the truth lemma, and finally obtain the completeness. 
%\end{proof}

Unfortunately, $\MLMSK$ over (constant-domain) S5 models is indeed too powerful, we can code first-order formulas by replacing each quantifier in a first-order formula in the prenex form by $\KS$ or $\hKS\Box$ respectively. We can show that this translation preserves the satisfiability.  For example, we can translate $\exists x \forall y \phi$ ($\phi$ is quantifier-free) into an $\MLMSK$ formula $\KS \hKSy\Box \phi$, which is equivalent to the first-order modal formula $\exists x \Box\forall y \Box \phi$ over S5 models, and it implies $\exists x\forall y \phi$ by reflexivity. If a first-order formula is satisfiable then we can build a single-world S5 model such that the translated $\MLMSK$-formula is also satisfiable. On the other hand, if the translated $\MLMSK$-formula is satisfiable in some pointed $\MLMSK$ S5 model, then we can just pick the designated world in that model as a first-order structure to satisfy the original first-order formula. This leads to the  undecidability of $\MLMSK$ due to the undecidability of first-order logic:
\begin{theorem}
$\MLMSK$ is undecidable over S5 models. 
\end{theorem}

\section{Future work}
%\noteYW{Varying domain, the logic is different: $\KS\phi\lra \Box \phi$ does not hold. But we can give a intuitive semantics to it even in multi-agent setting. We can require at least the possibilities share a common element. It does not help with the undecidability. We conjecture that van-benthem characterization theorem should hold also for finite models.}

Due to limited space, we omit several proofs and some additional results, and leave them to the extended version of this conference paper.
%\footnote{In the full version, besides the detailed proofs, we will give a simple complete axiomatization of $\MLMSEQ$ over arbitrary models. 
%We will also show that if $\MLMS$ has an increasing domain model then it also has a constant-domain model. 
%The undecidability of the satisfiability problem of $\MLMSKEQ$ over S5 models will be proved by coding the satisfiability problem of \FOL.
%} 
We believe that this is only the beginning of an interesting story. On the technical side, we may study potential properties of $\MLMSEQ$ such as interpolation, frame definability, characterization over finite models, axiomatization and (un)decidability on various frame classes (with or without extra constant symbols). Although the full $\MLMSK$ over S5 models is undecidable, the concrete ``propositional'' know-wh logics mentioned in the introduction are usually decidable. One explanation is that the existing know-wh logics are often similar to one-variable fragments of the first-order modal language, which may lead to decidability over S5 and other models \cite{DeFOML}, e.g., the conditional  \textit{knowing value} formulas  $\Kv(\phi,c)$ discussed in \cite{WF13,WF14} are essentially $\exists x \Box(\phi\to c=x)$. A detailed comparison with known guarded-like fragments with extra frame constraints will be very useful to understand the new framework more deeply. Moreover, we believe our techniques and results can be generalized to poly-modal (multi-agent) settings. It then makes sense to discuss the mention-some version of common knowledge operator $C^x:=\exists x C$ where $C$ is the propositional common knowledge operator.\footnote{This is a rather strong notion of common knowledge, e.g., we commonly know how to prove the theorem in this sense means a fixed proof is also commonly known.} There is also a clear similarity with modal logic over neighbourhood models, which is worth exploring.

\bibliographystyle{eptcs}
\bibliography{BKTALL}

\end{document}